\documentclass[11pt]{article}
\usepackage{ifthen}
\usepackage{mdwlist}
\usepackage{amsmath,amssymb,amsfonts,amsthm}
\usepackage{bm}
\usepackage{bbm}
\usepackage{enumitem}
\usepackage{graphicx}
\usepackage{xspace}
\usepackage{verbatim}
\usepackage{algorithm}
\usepackage{algpseudocode}
\usepackage[margin=1in]{geometry}
\usepackage{color}
\usepackage{thm-restate}
\usepackage{latexsym}
\usepackage{epsfig}
\usepackage{lipsum}
\usepackage{titlesec}
\usepackage{natbib}
\usepackage[colorlinks=true,linkcolor=blue,citecolor=blue,urlcolor=blue]{hyperref}
\usepackage{cleveref}
\usepackage[colorinlistoftodos,textsize=scriptsize]{todonotes}
\def\withcolors{1}
\def\withnotes{1}

\renewcommand{\epsilon}{\ve}
\def\ve{\varepsilon}

\newcommand{\pr}[2][]{\mathrm{Pr}\ifthenelse{\not\equal{}{#1}}{_{#1}}{}\!\left[#2\right]}

\newtheorem{theorem}{Theorem}

\newtheorem{lemma}[theorem]{Lemma}

\ifnum\withcolors=1
  \newcommand{\gcolor}[1]{{\color{red}#1}} 
\else

  \newcommand{\gcolor}[1]{{#1}}
\fi

\ifnum\withnotes=1
  \newcommand{\gnote}[1]{\par\gcolor{\textbf{G: }\sf #1}} 
  \newcommand{\gfootnote}[1]{\footnote{{\bf \gcolor{Gautam}}: {#1}}}

\else
  \newcommand{\gnote}[1]{}
  \newcommand{\gfootnote}[1]{}

\fi
\newcommand{\ignore}[1]{\leavevmode\unskip} 

\allowdisplaybreaks

\title{An Information-Theoretic Analysis of OOD Generalization in Meta-Reinforcement Learning}

\author {
Xingtu Liu\thanks{Simon Fraser University. {\tt rltheory@outlook.com}.}
}

\begin{document}
\maketitle

\begin{abstract}
In this work, we study out-of-distribution (OOD) generalization in meta-reinforcement learning from an information-theoretic perspective. We begin by establishing OOD generalization bounds for meta-supervised learning under two distinct distribution shift scenarios: standard distribution mismatch and a broad-to-narrow training setting. Building on this foundation, we formalize the generalization problem in meta-reinforcement learning and establish fine-grained generalization bounds that exploit the structure of Markov Decision Processes. Lastly, we analyze the generalization performance of a gradient-based meta-reinforcement learning algorithm.
\end{abstract}

\section{Introduction}
Meta-learning \citep{thrun1998learning,hospedales2021meta}, or ``learning to learn", is the study of algorithms that leverage prior experience across tasks to rapidly adapt to new tasks.  By formalizing the process of learning from prior tasks, meta-learning provides a principled framework for developing models that are flexible, data-efficient, and robust to new environments. Since the goal of meta-learning is to generalize to new or even unseen tasks, it is crucial to develop a theoretical understanding of its generalization behavior. In reality, training tasks often suffer from selection bias and distribution shift over time. This challenge, known as out-of-distribution (OOD) generalization, raises fundamental questions about how prior experience should be leveraged when the testing environment deviates from the training environment. In this work, we address these questions by leveraging information-theoretic tools.

Building on our analysis of meta-supervised learning, we extend the study of OOD generalization to meta-reinforcement learning (meta-RL) \citep{gupta2018meta, nagabandi2018learning, yu2020meta, beck2025tutorial}. Meta-RL adapts the meta-learning framework to sequential decision-making, where tasks are modeled as Markov decision processes (MDPs) and the task-specific objective is the cumulative reward. In real-world deployments (e.g., robotics, human interaction), agents often face shifts in rewards, dynamics, or state distributions beyond those seen in meta-training. Thus, understanding OOD generalization in meta-RL is crucial for ensuring reliable adaptation and decision-making \citep{beck2025tutorial}. Prior works such as \citet{rimon2022meta} and \citet{tamar2022regularization} analyzed in-distribution generalization in meta-RL, while \citet{simchowitz2021bayesian} provides a PAC-Bayesian analysis for Bayesian meta-RL. In contrast, our work develops a general information-theoretic analysis for OOD generalization in meta-RL and can be applied to subtask problems and gradient-based meta-RL algorithms. A key advantage of our framework lies in its ability to decouple the sources of OOD shift. Prior PAC-Bayesian analyses \citet{simchowitz2021bayesian} bound generalization error using the total variation distance between global priors, which treats the whole environment as a black box. Our approach leverages the underlying structure of the MDPs, explicitly capturing the divergence between transition kernels, initial state distributions, and rewards.

Following the seminal work of \citet{russo2016controlling} and \citet{xu2017information}, an information-theoretic framework has been developed to bound the generalization error of learning algorithms using the mutual information between the input dataset and the output hypothesis. This methodology formalizes the intuition that overfitted learning algorithms are less likely to generalize effectively. Unlike classical complexity-based approaches such as VC dimension or Rademacher complexity, the information-theoretic framework captures all aspects of the learning process, including the data distribution, hypothesis space, and learning algorithm. We therefore employ this framework to study generalization in meta-(reinforcement) learning under shifts between training and testing task distributions. 
In particular, we make the following contributions.

\begin{itemize}
    \item In Section \ref{subsec_ml1}, we study the distribution mismatch setting, where the test environment differs from the training environment. In this case, we derive an upper bound $O\left(\text{D}+\frac{\text{MI}}{nm}\right)$ on the OOD generalization error, where D quantifies the environment mismatch, MI denotes the mutual information, $n$ is the number of training tasks, and $m$ is the number of samples per task.
    \item In Section \ref{subsec_ml2}, we consider the subtask setting, in which the meta-algorithm is trained on a broader environment. Here, the test environment is assumed to consist of tasks that form a strict subset of the training environment. The resulting generalization bound indicates that merely increasing the number of training tasks that are weakly related to the targeting tasks provides limited benefit.
    \item In Section \ref{sec_mrl}, we extend our analysis to meta-RL, where both the hypothesis space and the objective differ from meta-supervised learning. We first extend the generalization bounds in Section \ref{subsec_ml1} with a finer-grained analysis. The divergence term captures the distributional shifts across the initial state, transition dynamics, and reward function. To address the potential unboundedness issue of this divergence, we extend the analysis from Section \ref{subsec_ml2} to the meta-RL setting. Finally, we derive a generalization bound for a two-level gradient-based meta-RL algorithm.
\end{itemize}

\section{Problem Formulation}

This section introduces preliminaries and problem formulation for meta-learning and meta-RL.

\subsection{Preliminaries}
Random variables are denoted by capital letters (e.g., $X$ and $Y$), and their realizations are denoted by lowercase letters (e.g., $x$ and $y$). For random variables $X$ and $Y$, $P_{X,Y}$ denotes their joint distribution, $P_X$ denotes the marginal distribution, and $P_{X|Y}$ denotes the conditional distribution. A random variable $X$ is said to be $\sigma$-sub-Gaussian if $\mathbb{E}[e^{\lambda(X-\mathbb{E}[X])}]\leq e^{\lambda^2\sigma^2/2}$ for all $\lambda\in \mathbb{R}$. For probability measures $\mu$ and $\nu$, we use $D(\mu \| \nu)$ to denote the Kullback-Leibler (KL) divergence of $\mu$ with respect to $\nu$. The mutual information between $X$ and $Y$ is defined as: $I(X; Y) := D({P}_{X,Y} \| {P}_X \otimes {P}_Y).$ The conditional mutual information is defined as: $I(X; Y|Z) := \mathbb{E}_Z[D({P}_{X,Y|Z} \| {P}_{X|Z} \otimes {P}_{Y|Z})].$
We denote by $\mathrm{cov}(X)$ and $H(X)$ the covariance matrix and entropy of a random variable $X$.
For a matrix $M$, $\mathrm{det}(M)$ denotes its determinant.
We use big-$O$ notation to characterize the asymptotic upper bound of a function's growth rate.

In supervised learning, given a training dataset $Z=\{Z_i\}_{i=1}^n \sim \mu^{\otimes n}$, a learning algorithm $\mathcal{A}$ outputs a hypothesis $W = \mathcal{A}(Z)$ from the hypothesis space $\mathcal{W}$. Let $\ell: \mathcal{W} \times \mathcal{Z} \rightarrow \mathbb{R}$ denote the loss function. We define the empirical risk of $W$ as $L(W):=\frac{1}{n} \sum_{i=1}^n\ell(W,Z_i)$, and define the in-distribution population risk as $L_{\mu}(W):=\mathbb{E}_{Z\sim\mu}[\ell(W,Z)]$. The expected in-distribution generalization error is defined as $\mathbb{E}_{Z\sim\mu}[L_{\mu}(W) - L(W)]$. Note that $W$ is a function of $Z$. When the training distribution $\mu$ differs from the testing distribution $\nu$, we have the OOD population risk $L_{\nu}(W):=\mathbb{E}_{Z^{\prime}\sim\nu}[\ell(W,Z^{\prime})]$. The expected OOD generalization is then defined as $\mathbb{E}_{Z\sim\mu}[L_{\nu}(W) - L(W)]$. In the following sections, we formally define OOD generalization for meta-learning and meta-RL, which is more intricate than in supervised learning.

\subsection{Meta-Learning}
In this work, we consider learning tasks sampled i.i.d. from the training environment $\mathcal{T}$. For a given task $\tau_i \sim \mathcal{T}$, let $Z_i = \{S_{i,j}\}_{j=1}^m \sim {\tau_i}^{\otimes m}$ denote its in-task training dataset. We assume each task is associated with a dataset of the same size $m$. Suppose there are $n$ training tasks. We denote the meta-training dataset by $Z_{1:n}=\{Z_i\}_{i=1}^n$ and the task-specific parameters by $W=\{W_i\}_{i=1}^n$. The goal of meta-learning is to learn a meta-parameter $\theta$ that captures knowledge shared across tasks. Given $\theta$ and a new task, the meta-learner can adapt more efficiently and produce a task-specific hypothesis. During training, the meta-learner aims to minimize the empirical meta-risk
\begin{align*}
L_{Z_{1:n}}(\theta) := \frac{1}{n} \sum_{i=1}^n \mathbb{E}_{W_i\sim P_{W_i|Z_i,\theta}} [L(W_i, Z_i)] := \frac{1}{n} \sum_{i=1}^n \mathbb{E}_{W_i\sim P_{W_i|Z_i,\theta}} \left[\frac{1}{m} \sum_{j=1}^m \ell(W_i,S_{i,j}) \right]
\end{align*}
where $L(W)$ denotes the task-specific empirical risk associated with parameter $W$, and $P_{W|Z,\theta}$ represents the distribution over task-specific parameters induced by the base-learner. The population meta-risk with respect to $\mathcal{T}$ is defined as
\begin{align*}
L_{\mathcal{T}}(\theta) := \mathbb{E}_{\tau \sim \mathcal{T}} \mathbb{E}_{Z\sim \tau^{\otimes m}} \mathbb{E}_{W\sim P_{W|Z,\theta}} [L_{\tau}(W)] := \mathbb{E}_{\tau \sim \mathcal{T}} \mathbb{E}_{Z\sim \tau^{\otimes m}} \mathbb{E}_{W\sim P_{W|Z,\theta}} \left[ \mathbb{E}_{S\sim \tau} \ell(W,S) \right].
\end{align*}
The quality of the learned meta-parameter $\theta$ can be evaluated via the (in-distribution) meta generalization error
\begin{align*}
\text{gen}_{\text{id}} := \mathbb{E}_{\theta,Z_{1:n}}[L_{Z_{1:n}}(\theta) - L_{\mathcal{T}}(\theta)].
\end{align*}
However, in practice, the testing environment typically differs from the training environment, and the essence of meta-learning is to enable fast adaptation to new tasks. Thus, we focus on the OOD meta generalization error. Let $\mathcal{U}$ denote the targeting task environment and let $\mu \sim \mathcal{U}$ represent a test task,
\begin{align*}
\text{gen}_{\text{ood}} := \mathbb{E}_{\theta,Z_{1:n}}[L_{Z_{1:n}}(\theta) - L_{\mathcal{U}}(\theta)]
\end{align*}
where 
\begin{align*}L_{\mathcal{U}}(\theta)= \mathbb{E}_{\mu \sim \mathcal{U}} \mathbb{E}_{Z\sim \mu^{\otimes m}} \mathbb{E}_{W\sim P_{W|Z,\theta}} \left[ \mathbb{E}_{S\sim \mu} \ell(W,S) \right].
\end{align*}

\subsection{Meta-Reinforcement Learning}

In supervised learning, the goal is to find a hypothesis that minimizes the empirical risk, whereas in standard RL, the goal is to find a policy that maximizes the cumulative reward. In meta-RL, the learner seeks to learn a meta-algorithm $f_{\theta}$, parameterized by a meta-parameter $\theta$, which serves as an adaptation rule. Given a new task, $f_{\theta}$ produces a task-specific policy $\pi_{\phi}$ parameterized by $\phi$ that aims to maximize the cumulative reward for that task.

The tasks in meta-RL are modeled as MDPs. An MDP is defined by $\mathcal{M}(\mathcal{S}, \mathcal{A}, \mathcal{P}, \rho, r, \gamma,H)$. We use $\Delta(\mathcal{X})$ to denote the set of the probability distribution over the set $\mathcal{X}$. $\mathcal{M}(\mathcal{S}, \mathcal{A}, \mathcal{P}, \rho, r, \gamma,H)$ is specified by a finite state space $\mathcal{S}$, a finite action space $\mathcal{A}$, transition function $\mathcal{P}: \mathcal{S} \times \mathcal{A} \rightarrow \Delta (\mathcal{S})$, initial state distribution $\rho$, reward function $r:\mathcal{S} \times \mathcal{A} \rightarrow \mathbb{R}$, discount factor $\gamma \in [0,1]$, and the horizon $H$. In this work, we consider a finite horizon for simplicity. We assume the reward is bounded, i.e. $r(s,a) \in [0,1]$\footnote{For rewards in $[R_{\min}, R_{\max}]$ simply rescale these bounds.}, $\forall (s,a)$. 

Assume there are $n$ MDPs $\{\mathcal{M}_i\}_{i=1}^n$ sampled i.i.d. from the training environment $\mathcal{T}$. Let $\phi=\{\phi_i\}_{i=1}^n$ denote the task-specific parameters. A key distinction between meta-RL and meta-supervised learning is that, in RL, we do not collect hypothesis-independent datasets. Instead, the data for each task consists of trajectories whose distribution depends on the task-specific parameters. Specifically, given $\phi$ and an MDP $\mathcal{M}$, a trajectory $\omega=\{s_h,a_h,r_h,s_{h+1}\}_{h=0}^H$ is distributed according to
\begin{align}
P_{\omega|\mathcal{M},\phi} = \rho(s_0) \prod_{h=0}^H \pi_{\phi}(a_h|s_h) \mathcal{P}(s_{h+1}|s_h,a_h) \label{eqn_trajP}
\end{align}
where $\pi_{\phi}$ is the task-specific policy, and $\rho$ and $\mathcal{P}$ are the initial state distribution and transition kernel of $\mathcal{M}$. The meta-learner aims to maximize the empirical meta-RL objective
\begin{align}
J_{\mathcal{M}_{1:n}}(\theta) := \frac{1}{n} \sum_{i=1}^n \mathbb{E}_{\phi_i\sim P_{\phi|\mathcal{M}_i,\theta}} [J(\pi_{\phi_i}, \mathcal{M}_i)] := \frac{1}{n} \sum_{i=1}^n \mathbb{E}_{\phi_i\sim P_{\phi|\mathcal{M}_i,\theta}} \left[ \mathbb{E}_{\omega \sim P_{\omega|\mathcal{M}_i,\phi_i}}\left[\sum_{j=0}^H \gamma^{j}r_j\right] \right] \label{eqn_generalmrl}
\end{align}
where $J(\pi_{\phi_i})$ denotes the expected discounted return of policy $\pi_{\phi}$. Let $\mathcal{U}$ denote the targeting (testing) task environment. The population meta-RL objective is defined as
\begin{align*}
J_{\mathcal{U}}(\theta) :=  \mathbb{E}_{\mathcal{M}\sim \mathcal{U}} \mathbb{E}_{\phi\sim P_{\phi|\mathcal{M},\theta}} [J_{\mathcal{M}}(\pi_{\phi})] :=  \mathbb{E}_{\mathcal{M}\sim \mathcal{U}} \mathbb{E}_{\phi\sim P_{\phi|\mathcal{M},\theta}} \left[ \mathbb{E}_{\omega \sim P_{\omega|\mathcal{M},\phi}}\left[\sum_{j=0}^H \gamma^{j}r_j\right] \right].
\end{align*}
Finally, the OOD meta-RL generalization error is given by
\begin{align*}
\text{gen}_{\text{ood}} := \mathbb{E}_{\theta,\mathcal{M}_{1:n}}[J_{\mathcal{M}_{1:n}}(\theta) - J_{\mathcal{U}}(\theta)].
\end{align*}

\section{Out-of-Distribution Generalization in Meta-Learning}
\label{sec_ml}

In this section, we investigate OOD generalization in meta-learning under two settings. The first is the mismatch setting, where the training environment $\mathcal{T}$ differs from the testing environment $\mathcal{U}$. The second is the subtask setting, where the tasks comprising $\mathcal{U}$ form a strict subset of those in $\mathcal{T}$. This latter case corresponds to a common meta-learning strategy which trains across a wide range of tasks to distill shared knowledge. For the first setting, we use mutual information to derive bounds on the OOD generalization error, while for the second setting, we apply conditional mutual information.

\subsection{OOD Generalization in Meta-Learning via Mutual Information}
\label{subsec_ml1}

We now present the main result. Let $P_{Z_{1:n}}$ denote the distribution of the meta-dataset under the training environment $\mathcal{T}$, and $Q_{Z_{1:n}}$ the corresponding distribution under the testing environment $\mathcal{U}$.

\begin{restatable}{theorem}{thmmlo}
\label{thm_ml1}
Suppose the loss function $\ell$ is $\sigma$-sub-Gaussian for any meta parameter $\theta$, hypothesis $W_i$, and dataset $Z_i$. The OOD meta generalization error is upper-bounded by
\begin{align*}
\mathbb{E}_{\theta,Z_{1:n}}[L_{Z_{1:n}}(\theta)-L_{\mathcal{U}}(\theta)] \leq \sqrt{\frac{2\sigma^2(I(\theta,W_{1:n};Z_{1:n})+D(P_{Z_{1:n}}\|Q_{Z_{1:n}}))}{nm}}.
\end{align*}
\end{restatable}

When the training and testing environments are identical, we have $D(P_{Z_{1:n}}\|Q_{Z_{1:n}})=0$, and the bound reduces to the standard in-distribution generalization bound. When the sampled tasks are i.i.d., we have $\frac{D(P_{Z_{1:n}}\|Q_{Z_{1:n}})}{nm}=\frac{D(P_{Z}\|Q_{Z})}{m}$. Moreover, when the in-task samples are i.i.d., it follows that $\frac{D(P_{Z}\|Q_{Z})}{m}=D(\tau\|\mu)$. Thus, strong generalization guarantees hold only if the distribution mismatch term is sufficiently small. By the chain rule of mutual-information, the bound can be further written as
\begin{align*}
\mathbb{E}_{\theta,Z_{1:n}}[L_{Z_{1:n}}(\theta)-L_{\mathcal{U}}(\theta)] \leq \sqrt{2\sigma^2D(\tau\|\mu)} + \sqrt{\frac{2\sigma^2I(\theta;Z_{1:n})}{nm}} + \sqrt{\frac{2\sigma^2 \sum_{i=1}^nI(W_{i};Z_{i}|\theta)}{nm}}.
\end{align*}
The first term accounts for the distribution mismatch, the second term reflects the environmental uncertainty, and the third term reflects the task-level uncertainty. 

Instead of mutual information, one may also upper bound the generalization error using the Wasserstein or total variation distance when the loss function is bounded or Lipschitz continuous \citep{lopez2018generalization,wang2022information,liu2025information}. We leave this extension for future work.

\subsection{Subtask Generalization in Meta-Learning via Conditional Mutual Information}
\label{subsec_ml2}

We now consider another setting in which the model is trained on a broad set of tasks but tested on a specific subset of them. This ``broad-to-narrow" scenario results in a distribution shift arising from transitioning from a broad environment to a narrower one. We denote $\hat{\mathcal{T}}$ as the testing environment consisting of fewer tasks than $\mathcal{T}$. In this case, one could in principle apply the previous results by viewing $\mathcal{U}$ as $\hat{\mathcal{T}}$. However, in this case, the term ${D(\tau|\mu)}$ introduces a persistent bias, regardless of the sample size. To address this issue, we derive a new upper bound based on conditional mutual information \citep{steinke2020reasoning,laakom2024class}. 

The conditional mutual information approach was introduced by \citet{steinke2020reasoning}, which normalizes the information content of each data point. Let ${{Z}}= \{Z^{\pm}_i\}_{i=1}^n$ consist of $2n$ samples drawn independently from $P_{Z}$. Let $U = \{U_i\}_{i=1}^n \in \{-1,1\}^n$ be uniformly random and independent from ${{Z}}$. Denote $Z_i^{U_i}$ as the training sample selected from $Z^{\pm}_i$, and denote $Z_i^{-U_i}$ as the testing sample. Similarly, let ${{W}}= \{W^{\pm}_i\}_{i=1}^n$ denote the collection of task-specific hypotheses obtained by training on $Z$ given the meta-parameter $\theta$.  Next, we define the subtask meta generalization error \citep{laakom2024class} as
\begin{align}
\text{gen}_{\text{sub}} &:=  \mathbb{E}_{\hat{\tau}\sim\hat{\mathcal{T}}} \mathbb{E}_{\theta,Z_{1:n}}\left[ \frac{1}{n_{\hat{\tau}}} \sum_{i=1}^n \mathbb{E}_{U_i}  \left[\mathbbm{1}{\{\tau_{i}^{U_i}=\hat{\tau}\}} \mathbb{E}_{W_i^{U_i}|\theta,Z_{i}^{U_i}}L(W_i^{U_i},Z_{i}^{U_i})\right.\right. \nonumber \\
&\left.\left. \indent - \mathbbm{1}{\{\tau_{i}^{-U_i}=\hat{\tau}\}} \mathbb{E}_{W_i^{-U_i}|\theta,Z_{i}^{-U_i}}L(W_i^{-U_i},Z_{i}^{-U_i})  \right] \right] \label{gen_sub}
\end{align}
where $L(W_i^{U_i},Z_{i}^{U_i})=\frac{1}{m}\sum_{j=1}^m \ell(W_i^{U_i},S^{U_i}_{i,j}).$

The subtask generalization error $\text{gen}_{\text{sub}}$ measures the expected difference in empirical training performance when adapting the meta-parameter to either of the paired datasets ($Z_i^+$ or $Z_i^-$), restricted to the target subtask environment. It quantifies how sensitive the empirical performance is to the specific data sample drawn for a task within the target subtask distribution. This definition is similar to those in the super-sample setting \citep{zhou2022individually}, but is adapted for meta-learning.

\begin{restatable}{theorem}{thmmlt}
\label{thm_ml2}
Suppose the loss function $\ell$ is $\sigma$-sub-Gaussian. The subtask meta generalization error defined in \cref{gen_sub} is upper-bounded by
\begin{align*}
\textnormal{gen}_{\textnormal{sub}} \leq \mathbb{E}_{\hat{\tau}\sim \hat{\mathcal{T}}} \mathbb{E}_{Z_{1:n}}\left[ \frac{1}{n_{\hat{\tau}}} \sum_{i=1}^n \sqrt{\frac{2\sigma^2 I(f_i;U_i|Z_{1:n})}{m}} \right]
\end{align*}
where
\begin{align*}
f_i := \mathbbm{1}{\{\tau_{i}^{-}=\hat{\tau}\}} \mathbb{E}_{W_i^{-}|\theta,Z_{i}^{-}}L(W_i^{-},Z_{i}^{-})  - \mathbbm{1}{\{\tau_{i}^{+}=\hat{\tau}\}} \mathbb{E}_{W_i^{+}|\theta,Z_{i}^{+}}L(W_i^{+},Z_{i}^{+}).
\end{align*}
\end{restatable}

The conditional mutual information $I(f_i;U_i|Z_{1:n})$ quantifies how much information the choice of dataset ($Z_i^+$ vs. $Z_i^-$) provides about the resulting difference in expected empirical performance $f_i$ on the target task. A lower conditional mutual information suggests a more stable adaptation process relevant to the target task.

This bound reflects the intuition behind meta-learning that more data per task helps. When $m$ is sufficiently large, the contribution of meta-information becomes less significant. Besides, if the number of training tasks $n$ increases without a corresponding rise in tasks that are closely related to the targeting distribution, the generalization bound may become even worse, as $n$ grows while $n_{\hat{\tau}}$ remains fixed. 
In such cases, although the averaged mutual information $\frac{1}{n} \sum_{i=1}^n \sqrt{I(f_i;U_i|Z_{1:n})}$ may decrease, the term $\frac{1}{n_{\hat{\tau}}} \sum_{i=1}^n \sqrt{I(f_i;U_i|Z_{1:n})}$ can increase if the trade-off between enlarging the training set and improving the meta-parameter is not properly balanced.
In contrast, when additional training tasks are collected and the number of samples per class of task increases proportionally, the model can better distill shared structure across tasks, allowing the meta-information to help generalization.

\section{Out-of-Distribution Generalization in Meta-Reinforcement Learning}
\label{sec_mrl}

Meta-RL provides a framework for learning transferable priors across RL tasks. Building on our OOD generalization analysis in meta-supervised learning, we now extend these results to meta-RL. The meta-RL setting introduces two key distinctions: the objective shifts from empirical loss minimization to maximizing expected cumulative reward, and the data is no longer a fixed, hypothesis-independent dataset. Instead, it consists of trajectories generated through the agent’s own interactions with the MDPs, making the data distribution inherently dependent on the learning parameters.
We begin in Section \ref{subsec_mrl1} by deriving general OOD bounds for meta-RL under both the distribution mismatch and subtask settings. We then apply this framework in Section \ref{subsec_mrl3} to analyze the generalization performance of a specific gradient-based meta-RL algorithm. Finally, in Section \ref{sec_43}, we demonstrate how these generalization error bounds can be directly related to the suboptimality gap in the target task environment, considering both standard and offline meta-RL settings.

\subsection{Generalization Bounds for Meta-Reinforcement Learning}
\label{subsec_mrl1}

We first present the following result, which can be viewed as an application of Theorem \ref{thm_ml1} to the meta-RL setting.

\begin{restatable}{theorem}{thmmrlo}
\label{thm_mrrl1}
The OOD meta-RL generalization error is upper-bounded by
\begin{align*}
\mathbb{E}_{\theta,\mathcal{M}_{1:n}}[J_{\mathcal{M}_{1:n}}(\theta)-J_{\mathcal{U}}(\theta)] \leq \sqrt{\frac{2I(\theta,\phi_{1:n};\mathcal{M}_{1:n})+2D(P_{\mathcal{M}_{1:n}}\|Q_{\mathcal{M}_{1:n}})}{n(1-\gamma)^2}}.
\end{align*}
\end{restatable}
As before, this bound can be further decomposed as follows

\begin{align}
\mathbb{E}_{\theta,\mathcal{M}_{1:n}}[J_{\mathcal{M}_{1:n}}(\theta)-J_{\mathcal{U}}(\theta)] \leq \sqrt{\frac{2D(P_{\mathcal{M}}\|Q_{\mathcal{M}})}{(1-\gamma)^2}} + \sqrt{\frac{2I(\theta;\mathcal{M}_{1:n})}{n(1-\gamma)^2}} + \sqrt{\frac{2 \sum_{i=1}^nI(\phi_{i};\mathcal{M}_{i}|\theta)}{n(1-\gamma)^2}}. \label{eqn_mrld}
\end{align}

The first term measures the distribution mismatch between the training and target environments, which vanishes when there is no distribution shift. This term can be further decomposed into three parts, which measure the divergence between the initial state distribution, transition kernel, and reward.

\begin{restatable}{lemma}{lemdcom}
\label{lem_decompose}
The KL divergence between two environments decomposes into the respective KL divergences of their initial state distributions, transition kernels, and reward functions:
\begin{align*}
D(P_{\mathcal{M}}\| Q_{\mathcal{M}}) \leq D(\rho\| \rho^{\prime}) + (H+1) \max_{(s,a)\in \mathcal{S}\times \mathcal{A}}( D(\mathcal{P}(\cdot|s,a)\| \mathcal{P}^{\prime}(\cdot|s,a)) + D(r(s,a)\|r^{\prime}(s,a))).
\end{align*}
\end{restatable}

The decomposition established in this lemma highlights a fundamental advantage of our approach over existing PAC-Bayesian bounds \citet{simchowitz2021bayesian}, which relies on the total variation distance to measure the discrepancy between latent priors. Because total variation lacks a natural chain rule, those derived bounds inherently couple all sources of environmental variation together. In contrast, our KL-divergence based formulation exploits the underlying structure of the MDPs. This structural awareness allows our bounds to pinpoint exactly which component of the distribution shift is driving the generalization error. Thus, it avoids unnecessary pessimism when only a subset of the MDP components shifts.

The second term in \cref{eqn_mrld} captures the environmental uncertainty, reflecting how strongly the learned meta-parameters depend on the particular set of training environments. This term decreases as the number of training tasks increases, indicating improved robustness with more training tasks.

The third term in \cref{eqn_mrld} quantifies the information each adapted task-level policy $\phi_i$ retains about its corresponding MDP $\mathcal{M}_i$ given the shared meta-parameter. A small mutual information implies that the policy primarily focuses on maximizing rewards rather than memorizing task-specific details. Unfortunately, as in meta-learning, this term remains nonzero even with infinitely many training tasks, since each task inherently involves adaptation uncertainty conditioned on $\theta$. This residual bias explains why meta-RL cannot be perfectly unbiased, even when there is no distribution shifts. Finally, the discount factor $\gamma$ amplifies all these effects, as tasks with longer effective horizons are more sensitive to distributional and uncertainty-related deviations.

Note that the divergence terms in Lemma \ref{lem_decompose} can be potentially unbounded. To avoid such issue, we now consider the subtask problem in meta-RL. We define the subtask meta-RL generalization error as
\begin{align}
\text{gen}_{\text{sub}} &:=  \mathbb{E}_{\hat{\mathcal{M}}\sim\hat{\mathcal{T}}} \mathbb{E}_{\theta,\mathcal{M}_{1:n}}\left[ \frac{1}{n_{\hat{\mathcal{M}}}} \sum_{i=1}^n \mathbb{E}_{U_i}  \left[\mathbbm{1}{\{\mathcal{M}_{i}^{U_i}=\hat{\mathcal{M}}\}} \mathbb{E}_{\phi_i^{U_i}|\theta,\mathcal{M}_{i}^{U_i}}J(\pi_{\phi_i^{U_i}},\mathcal{M}_{i}^{U_i})\right.\right. \nonumber \\
&\left.\left. \indent - \mathbbm{1}{\{\mathcal{M}_{i}^{-U_i}=\hat{\mathcal{M}}\}} \mathbb{E}_{\phi_i^{-U_i}|\theta,\mathcal{M}_{i}^{-U_i}}J(\pi_{\phi_i^{-U_i}},\mathcal{M}_{i}^{-U_i})  \right] \right] \label{gen_subRL}
\end{align}
where $J(\pi_{\phi_i^{U_i}},\mathcal{M}_{i}^{U_i})=\mathbb{E}_{\omega \sim P_{\omega|\mathcal{M}_i^{U_i},\phi_i^{U_i}}}\left[\sum_{j=0}^H \gamma^{j}r_j\right]$.

\begin{restatable}{theorem}{thmmrlt}
\label{thm_mrrl7}
The subtask meta-RL generalization error defined in \cref{gen_subRL} is upper-bounded by
\begin{align*}
{\textnormal{gen}_{\textnormal{sub}}} \leq \mathbb{E}_{\hat{\mathcal{M}}\sim \hat{\mathcal{T}}} \mathbb{E}_{\mathcal{M}_{1:n}}\left[ \frac{1}{n_{\hat{\mathcal{M}}}} \sum_{i=1}^n \sqrt{\frac{2 I(f_i;U_i|\mathcal{M}_{1:n})}{(1-\gamma)^2}} \right]
\end{align*}
where
\begin{align*}
f_i := \mathbbm{1}{\{\mathcal{M}_{i}^{-}=\hat{\mathcal{M}}\}} \mathbb{E}_{\phi_i^{-}|\theta,\mathcal{M}_{i}^{-}}J(\pi_{\phi_i^{-}},\mathcal{M}_{i}^{-})  - \mathbbm{1}{\{\mathcal{M}_{i}^{+}=\hat{\mathcal{M}}\}} \mathbb{E}_{\phi_i^{+}|\theta,\mathcal{M}_{i}^{+}}J(\pi_{\phi_i^{+}},\mathcal{M}_{i}^{+}).
\end{align*}
\end{restatable}

This bound for the subtask meta-RL setting provides an interpretation analogous to that of Theorem \ref{thm_ml2}. A low conditional mutual information implies that the learned meta-parameter provides a strong prior, making the adapted policy's performance less sensitive to the specific trajectories experienced during adaptation. Besides, simply increasing the total number of training tasks without a proportional increase in tasks relevant to the target environment may fail to improve generalization. Lastly, unlike Theorem \ref{thm_ml2}, this result focuses on task-level samples and does not involve the within-task sample size $m$.

\subsection{Generalization Bounds for Gradient-Based Meta-RL Algorithm}
\label{subsec_mrl3}

Many existing meta-RL methods, particularly black-box and task-inference approaches, struggle to generalize effectively outside their training distribution. While parameterized policy gradient methods like MAML have theoretical potential to adapt eventually as they retain the structure of a standard RL algorithm, a theoretical understanding of their performance is limited.

In this subsection, we analyze a variant of the meta-gradient RL algorithm, where both the policy gradient and meta-gradient updates are perturbed with isotropic Gaussian noise.
Meta-gradient RL \citep{xu2018meta,xu2020meta} optimizes task-specific and meta parameters through a two-level optimization process: an inner loop that updates task-specific parameters using the current meta parameter and collected trajectories, and an outer loop that updates the meta parameter after inner-loop updates. In this variant, both updates are performed using noisy gradients, where isotropic Gaussian perturbations are added to each step.
This approach can be viewed as an application of stochastic gradient Langevin dynamics (SGLD) \citep{welling2011bayesian} to the meta-RL setting. The injected noise helps the algorithm escape local minima and converge to global optima under sufficiently regular non-convex objectives. Additionally, the stochastic perturbations naturally encourage exploration in reinforcement learning \citep{ishfaq2025langevin}.

Suppose $\theta \in \mathbb{R}^d$, and the inner loop performs $T$ updates. Given a specific task $\mathcal{M}_i$, under MAML \citep{finn2017model}, the inner-loop is a policy gradient algorithm whose initial parameter $\phi_i^0$ is the meta-parameters $\theta$. At iteration $t$, the algorithm collects data 
\begin{align*}
\omega_{i,t}=\{s_h,a_h,r_h,s_{h+1}\}_{h=0}^H \sim \rho_i(s_0)\prod_{h=0}^H \pi_{\phi_i^t}(a_h|s_h) \mathcal{P}_i(s_{h+1}|s_h,a_h)
\end{align*}
where $\rho_i$ and $\mathcal{P}_i$ are associated with $\mathcal{M}_i$. We assume each step uses one trajectory of horizon $H$. For a trajectory $\omega_{i,t}$, the inner gradient estimator can be computed by
\begin{align*}
\nabla J(\pi_{\phi_i^t},\mathcal{M}_i)=\sum_{h=0}^{H-1} \gamma^h r_h \sum_{k=0}^{H-1}\nabla \log \pi_{{\phi}_i^t}(a_k|s_k).
\end{align*}
We then update the task parameter by
\begin{align*}
\phi_i^{t+1} =  \phi_i^{t} + \beta_t \nabla J(\pi_{\phi_i^t},\mathcal{M}_i) + \zeta_t
\end{align*}
where $\beta_t$ is the learning rate and $\zeta_t \sim \mathcal{N}(\mathbf{0},\kappa_t^2\mathbf{1}_d)$ is an isotropic Gaussian noise. Suppose the outer loop performs $M$ updates. At iteration $m$, we sample a batch $\mathcal{B}_m \subseteq \mathcal{M}_{1:n}$ of size $b$. The meta parameter is then updated as
\begin{align*}
\theta^{m+1}=\theta^m + \alpha_m \frac{1}{b} \sum_{\mathcal{M}_i\in \mathcal{B}_m} \nabla J_{\mathcal{M}_i}(\theta^m) +\xi_m
\end{align*}
where $\alpha_m$ is the learning rate and $\xi_m \sim \mathcal{N}(\mathbf{0},\tilde{\kappa}_m^2\mathbf{1}_d)$. Explicit formulations of the meta-gradient estimators can be found in \citet{al2017continuous,stadie2018some}. In the previous sections, we identified task-level and environmental uncertainties in the upper bound. Beyond improving exploration and optimization, the injected noise also help mitigating these two sources of uncertainty.

\begin{restatable}{theorem}{thmmrlff}
\label{thm_mrl5}
The OOD meta-RL generalization error for the noisy iterative meta-gradient RL algorithm is upper-bounded by
\begin{align*}
&\mathbb{E}_{\theta,\mathcal{M}_{1:n}}[J_{\mathcal{M}_{1:n}}(\theta)-J_{\mathcal{U}}(\theta)] \leq \sqrt{\frac{2D(P_{\mathcal{M}_{1:n}}\|Q_{\mathcal{M}_{1:n}})+\mathcal{E}_1+\mathcal{E}_2}{n(1-\gamma)^2}}
\end{align*}
where
\begin{align*}
&\mathcal{E}_1 = \sum_{m=0}^{M-1} \mathbb{E}_{\theta^m}\left[ \log\left(\textnormal{det}\left(\frac{\alpha_m^2}{\tilde{\kappa}^2_m}\tilde{\Sigma}_m+\mathbf{1}_d\right)\right)\right],\\
&\mathcal{E}_2=\sum_{m=0}^{M-1}\sum_{t=1}^T \mathbb{E}_{\theta^m,\phi_{1:n}^{T(m+1)-t}}\left[\log\left(\textnormal{det}\left(\frac{\beta_{T(m+1)-t}^2}{{\kappa}^2_{T(m+1)-t}}{\Sigma}_{T(m+1)-t}+\mathbf{1}_{nd}\right)\right)\right], \\
&\tilde{\Sigma}_m = \textnormal{cov}\left(\frac{1}{b} \sum_{\mathcal{M}_i\in \mathcal{B}_m} \nabla J_{\mathcal{M}_i}(\theta^m)\right), \quad \text{and} \quad {\Sigma}_k= \textnormal{cov}\left(\begin{bmatrix}
    \nabla J(\pi_{\phi_{1}^k},\mathcal{M}_1) \\
    \vdots \\
    \nabla J(\pi_{\phi_{n}^k},\mathcal{M}_n)
\end{bmatrix}\right).
\end{align*}
\end{restatable}

We have established an OOD generalization bound for a noisy iterative meta-gradient RL algorithm. The term $\mathcal{E}_1$ reflects environmental uncertainty and depends on the meta-learning rate, the injected noise variance, and the covariance of the meta-gradients. The term $\mathcal{E}_2$ reflects task-level uncertainty and depends on the inner-loop learning rate, the noise variance, and the covariance of the inner-loop gradients.
The SGLD algorithm analyzed in \citet{chen2021generalization,liu2025information} assumes bounded gradients. However, policy gradients are often unbounded in practice. Thus, we incorporate the gradient variances $\tilde{\Sigma}_m$ and ${\Sigma}_k$ into the upper bound \citep{wen2025towards}, which captures the sharpness of the optimization landscape \citep{jiang2019fantastic}.

\subsection{From Generalization to Suboptimality}
\label{sec_43}

We now demonstrate how the meta-RL generalization error can be related to the suboptimality gap in the target task environment.
Let $\text{gen}_{\text{ood}}(\theta)$ denote the OOD meta-RL generalization error associated with the meta-parameter $\theta$.
Let $\hat{\theta} = \mathcal{A}(\mathcal{M}_{1:n})$ be the meta-parameter output by a meta-learner that maximizes the empirical meta-RL objective.
Define $\theta^* \in \arg\max_{\theta} J_{\mathcal{U}}(\theta)$ as the optimal meta-parameter with respect to the population objective.
Then, by Theorem \ref{thm_mrrl1}, we have
\begin{align}
\mathbb{E}[J_{\mathcal{U}}(\theta^*) - J_{\mathcal{U}}(\hat{\theta})] &\leq \mathbb{E}[J_{\mathcal{U}}(\theta^*) - J_{\mathcal{M}_{1:n}}(\theta^*) + J_{\mathcal{M}_{1:n}}(\theta^*) - J_{\mathcal{M}_{1:n}}(\hat{\theta}) + J_{\mathcal{M}_{1:n}}(\hat{\theta}) - J_{\mathcal{U}}(\hat{\theta})] \nonumber \\
&\leq \text{gen}_{\text{ood}}(\hat{\theta}) - \text{gen}_{\text{ood}}(\theta^*). \label{regret_1}
\end{align}
Next, we turn to the discussion of the offline RL setting. Note that the analyzed generalization error thus far is task-level. The same framework can be extended to the offline RL setting, where the generalization error additionally accounts for within-task samples. In particular, for the offline meta-RL setting under episodic MDPs with horizon $H$, the expected discounted return in \cref{eqn_generalmrl} can be replaced by an unbiased estimator of the Bellman error. For a given task $\mathcal{M}_i$, let $\phi_i=(\phi_{i,j})_{j=1}^H$ denote the estimated optimal $Q$-value functions, and let $\pi_{\phi_i}$ denote the policy induced by $\phi_i$. The Bellman error is then defined as 
\begin{align*}
\mathcal{E}(\phi_i) = \frac{1}{H} \sum_{h=1}^H \| \phi_{i,h} - \mathcal{T}^*_h \phi_{i,h+1}  \|^2
\end{align*}
where $\mathcal{T}^*_h$ is the Bellman operator that propagates value functions forward by one step under the greedy policy.
Let $Z_i = \{ S_{i,j} \}_{j=1}^m$ denote the in-task training dataset for task $\mathcal{M}_i$. An unbiased estimator of the mean squared empirical Bellman error can be constructed using the double-sampling trick (i.e., sampling two next states $s^{\prime}$ and $s^{\prime \prime}$ for each transition) \citep{duan2021risk}
\begin{align*}
L(\phi_i,Z_i) := \frac{1}{mH} \sum_{(s,a,r,s^{\prime},s^{\prime \prime},h)\in Z_i} \left[ (\phi_{i,h}(s,a)-r-V_{\phi_{i,h+1}}(s^{\prime}))^2 -\frac{1}{2}(V_{\phi_{i,h+1}}(s^{\prime}) - V_{\phi_{i,h+1}}(s^{\prime \prime}))^2 \right].
\end{align*}
Under this setting, the meta-learner aims to maximize the empirical meta-RL objective
\begin{align*}
J_{Z_{1:n}}(\theta) := \frac{1}{n} \sum_{i=1}^n \mathbb{E}_{\phi_i\sim P_{\phi|Z_i,\theta}} [L(\phi_i,Z_i)],
\end{align*}
while the corresponding population meta-risk is defined as
\begin{align*}
J_{\mathcal{U}}(\theta) := \mathbb{E}_{\mathcal{M}\sim\mathcal{U}}\mathbb{E}_{Z\sim\mathcal{M}}\mathbb{E}_{\phi\sim P_{\phi|Z,\theta}} [\mathcal{E}(\phi)].
\end{align*}
Consequently, the OOD generalization error is given by $\text{gen}_{\text{ood}}=\mathbb{E}_{\theta,Z_{1:n}}[J_{Z_{1:n}}(\theta)-J_{\mathcal{U}}(\theta)]$. Since $L(\phi_i,Z_i) \in [-2H^2,4H^2]$, it follows under our proof framework that
\begin{align*}
\mathbb{E}_{\theta,Z_{1:n}}[J_{Z_{1:n}}(\theta)-J_{\mathcal{U}}(\theta)] \leq \sqrt{\frac{64H^2(I(\theta,\phi_{1:n};Z_{1:n})+D(P_{Z_{1:n}}\|Q_{Z_{1:n}}))}{nm}}.
\end{align*}
Let $V_{\mathcal{M}}^{*}(s_1)$ and $V_{\mathcal{M}}^{\pi}(s_1)$ denote, respectively, the optimal value function and the value function induced by policy $\pi$ at the initial state $s_1$ and step 1 in MDP $\mathcal{M}$.
By relating the Bellman error to value suboptimality \citep{duan2021risk}, we obtain
\begin{align}
\mathbb{E}_{\mathcal{M}\sim\mathcal{U}}\mathbb{E}_{Z\sim\mathcal{M}}\mathbb{E}_{\phi\sim P_{\phi|Z,\theta}} [V_{\mathcal{M}}^{*}(s_1)-V_{\mathcal{M}}^{\pi_{\phi}}(s_1)] \leq 2H\sqrt{C \cdot \text{gen}_{\text{ood}}} + 2H\sqrt{C \cdot\mathbb{E}_{\theta,Z_{1:n}}[J_{Z_{1:n}}(\theta)]} \label{regret_2}
\end{align}
where $C$ is the concentrability coefficient, which quantifies the adequacy of dataset coverage for off-policy evaluation. Thus, from \cref{regret_1} and \cref{regret_2}, we establish the connection between the meta-RL generalization error and the suboptimality gap in both standard and offline settings.

\section{Conclusion}

In this paper, we provided an information-theoretic analysis for OOD generalization in both meta-supervised learning and meta-RL. We investigated two distinct scenarios: the classical distribution mismatch setting and a ``broad-to-narrow" subtask setting. Our generalization bound for meta-RL explicitly decouples the distribution shift, measuring the divergence across the initial state distribution, transition kernel, and reward function. Besides, we derived a generalization bound for a gradient-based meta-RL algorithm and discussed the connection between the generalization error and the suboptimality gap in both standard and offline meta-RL.

\bibliographystyle{plainnat}
\bibliography{ref}

\appendix
\newpage
\tableofcontents

\section{Related Work}

\paragraph{Generalization Bounds for Meta-Learning.} Theoretical analyses of meta-learning trace back to \citet{baxter2000model}. More recently, many approaches have been used to establish generalization bounds, including PAC-Bayes \citep{amit2018meta,rothfuss2021pacoh, rezazadeh2022unified, zakerinia2024more}, uniform convergence \citep{tripuraneni2021provable, guan2022task, aliakbarpour2024metalearning}, and stability \citep{farid2021generalization, fallah2021generalization, chen2023stability, wang2024stability}. There are also works analyzing the excess risk of meta-learning in the contexts of convex optimization \citep{fallah2021generalization} and representation learning \citep{du2020few}. The most relevant works are \citet{jose2021information, chen2021generalization, rezazadeh2021conditional, hellstrom2022evaluated, wen2025towards}, which use (conditional) mutual information to bound generalization errors for meta-learning. However, they assume identical training and testing environments, whereas we focus on OOD generalization. The performance of meta-learning on OOD tasks has also been explored empirically \citep{li2018learning, lee2019learning, jeong2020ood, setlur2021two, chen2023secure, zhang2024rotogbml, hu2025task}.

\paragraph{Generalization Bounds for Meta-Reinforcement Learning.} Generalization bounds for meta-RL have been far less explored compared to meta-learning. While many works have studied the (OOD) generalization performance of meta-RL empirically \citep{kirsch2019improving, lee2021improving, ajay2022distributionally, bao2025toward, kim2025task}, theoretical analyses remain limited. \citet{rimon2022meta} propose a model-based approach and establish probably approximately correct (PAC) bounds on the number of training tasks required for learning an approximately Bayes-optimal policy. \citet{tamar2022regularization} also provide PAC bounds using algorithmic stability arguments. However, both works focus on in-distribution generalization. In contrast, \citet{simchowitz2021bayesian} study the OOD setting by analyzing the effect of misspecified priors for Thompson sampling. 

\paragraph{Information-Theoretic Generalization Bounds.} The information-theoretic approach was first introduced by \citet{russo2016controlling, xu2017information} and later refined to derive tighter bounds \citep{asadi2018chaining, hafez2020conditioning}. Since then, a wide range of tools have been developed, incorporating concepts such as conditional mutual information \citep{steinke2020reasoning}, $f$-divergence \citep{esposito2021generalization}, the Wasserstein distance \citep{lopez2018generalization, wang2019information}, and more \citep{aminian2021information, aminian2024learning}. In parallel, some works have focused on analyzing specific algorithms \citep{pensia2018generalization, negrea2019information, haghifam2020sharpened, NEURIPS2021_cf0d02ec, neu2021information, wang2023generalization}, while others have targeted particular settings such as deep learning \citep{he2024information}, federated learning \citep{wang2025generalization}, iterative semi-supervised learning \citep{he2021information}, transfer learning \citep{wu2020information}, batch reinforcement learning \citep{liu2024information}, and meta-learning \citep{jose2021information}. Of particular relevance are works that study distribution mismatch for supervised learning \citep{wu2020information, masiha2021learning, wang2022information, wang2024f, liu2025information}. There are also works that attempt to establish a unified framework for generalization from an information-theoretic perspective \citep{haghifam2021towards, chu2023unified, alabdulmohsin2020towards}. For a comprehensive overview, see the recent survey by \citet{hellstrom2025generalization}.

\section{Supporting Lemmas}

\begin{lemma}[Donsker-Varadhan Representation \citep{concenbook}]
\label{lem_1}
Let $P$ and $Q$ be two probability measures defined on a set $\mathcal{X}$ such that $P$ is absolutely continuous with respect to $Q$. For any bounded function $f: \mathcal{X} \rightarrow \mathbb{R}$ and $\mathbb{E}_{X\sim Q}[e^{f(X)}] \leq \infty$, we have
\begin{align*}
D(P\|Q) = \sup_{f} \left\{ \mathbb{E}_{X\sim P}[{f(X)}] - \log \mathbb{E}_{X\sim Q}[e^{f(X)}] \right\}.
\end{align*}
\end{lemma}

\begin{lemma}[Data Processing Inequality \citep{cover1999elements}]
\label{lem_2}
Given random variables $W, X, Y, Z$, and the Markov Chain $X\rightarrow Y \rightarrow Z$, we have $I(X;Z) \leq I(X;Y)$ and $I(X;Z) \leq I(Y;Z)$. For Markov chain $W \rightarrow X\rightarrow Y \rightarrow Z$, we have $I(X;Z|W) \leq I(X;Y|W)$ and $I(X;Z|W) \leq I(Y;Z|W)$.
\end{lemma}

\begin{lemma}[\citet{dong2024towards}]
\label{lem_ent}
Let $X \sim \mathcal{N}(\textbf{0},\Sigma)$ and $Y$ be any zero-mean random vector satisfying $\text{cov}[Y]=\Sigma$, then $H(Y)\leq H(X)$.
\end{lemma}

\begin{lemma}
\label{lem_3}
Let $P_{X, Y}=P_{Y|X}P_X$ be a joint probability distribution over random variables $X$ and $Y$. Let $Q_X$ be any other probability distribution over $X$. If a joint distribution $R_{X,Y}$ is defined as the product $R_{X,Y}=Q_X P_Y$, then the following identity holds.
\begin{align*}
D(P_{X,Y}\|R_{X,Y}) = I_P(X;Y) + D(P_X\|Q_X)
\end{align*}
where $I_P(X;Y)$ is the mutual information between $X$ and $Y$ under $P$.
\end{lemma}

\begin{proof}
By the definition of KL divergence,
\begin{align*}
D(P_{X,Y}\|R_{X,Y}) &= \mathbb{E}_{(x,y) \sim P_{X,Y}}\left[ \log \frac{P_{X,Y}(x,y)}{R_{X,Y}(x,y)} \right]\\
&= \mathbb{E}_{(x,y) \sim P_{X,Y}}\left[ \log \frac{P_{Y|X}(y)P_{X}(x)}{Q_{X}(x)P_{Y}(y)} \right]\\
&= \mathbb{E}_{(x,y) \sim P_{X,Y}}\left[ \log \frac{P_{Y|X}(y)}{P_{Y}(y)} \right] +  \mathbb{E}_{(x,y) \sim P_{X,Y}}\left[ \log \frac{P_{X}(x)}{Q_{X}(x)} \right]\\
&= \mathbb{E}_{(x,y) \sim P_{X,Y}}\left[ \log \frac{P_{Y|X}(y)}{P_{Y}(y)} \right] +  \mathbb{E}_{x \sim P_{X}}\left[ \log \frac{P_{X}(x)}{Q_{X}(x)} \right]\\
&=I_P(X;Y) + D(P_X\|Q_X).
\end{align*}
\end{proof}

\section{Proofs for Section \ref{sec_ml}} 

\subsection{Proof of Theorem \ref{thm_ml1}}
\thmmlo*

\begin{proof}
Let $\theta \in \Theta$ and $W_{1:n} \in \mathcal{W}^{n}$. Denote $P_{Z_{1:n}}$ as the distribution over the meta-dataset $Z_{1:n}$ that results from the training environment $\mathcal{T}$, and $Q_{Z_{1:n}}$  as the distribution over the meta-dataset $Z_{1:n}$ that results from the testing environment $\mathcal{U}$.
By Lemma \ref{lem_3} with $X=Z_{1:n}$, $Y=(\theta,W_{1:n})$, $P_{X}=P_{Z_{1:n}}$, and $Q_X=Q_{Z_{1:n}}$, it follows that
\begin{align*}
I(\theta, W_{1:n};Z_{1:n}) + D(P_{Z_{1:n}}\|Q_{Z_{1:n}}) = D(P_{\theta,W_{1:n},Z_{1:n}}\|P_{\theta,W_{1:n}}\otimes Q_{Z_{1:n}}).
\end{align*}
Recall that $Z_{1:n}=\{ Z_i \}_{i=1}^n=\{ S_{i,j} \}_{i,j=1}^{n,m}$. Let 
\begin{align*}
f(\theta, W_{1:n}, Z_{1:n}) := \frac{1}{n} \sum_{i=1}^n \frac{1}{m} \sum_{j=1}^m \ell(W_i,S_{i,j})
\end{align*}
Denoting $(\tilde{\theta},\tilde{W}_{1:n})$ as an independent copy of $(\theta,W_{1:n})$, we have
\begin{align*}
&D(P_{\theta,W_{1:n},Z_{1:n}}\|P_{\theta,W_{1:n}}\otimes Q_{Z_{1:n}}) \\
&\geq \sup_{\lambda} \left\{ \mathbb{E}_{\theta,W_{1:n},Z_{1:n} \sim P_{\theta,W_{1:n},Z_{1:n}}}[{\lambda f(\theta, W_{1:n}, Z_{1:n})}] - \log \mathbb{E}_{\tilde{\theta},\tilde{W}_{1:n} \sim P_{\theta,W_{1:n}},Z_{1:n} \sim Q_{Z_{1:n}}}[e^{\lambda f(\tilde{\theta},\tilde{W}_{1:n}, Z_{1:n})}] \right\} \tag{by Lemma \ref{lem_1}}\\
&= \sup_{\lambda} \left\{\lambda \big( \underbrace{\mathbb{E}_{\theta,W_{1:n},Z_{1:n} \sim P_{\theta,W_{1:n},Z_{1:n}}}[{f(\theta, W_{1:n}, Z_{1:n})}]}_{\text{Term (1)}} - \underbrace{\mathbb{E}_{\tilde{\theta},\tilde{W}_{1:n}\sim P_{\theta,W_{1:n}},Z_{1:n} \sim Q_{Z_{1:n}}}[{f(\tilde{\theta},\tilde{W}_{1:n}, Z_{1:n})}]}_{\text{Term (2)}} \big) \right. \\
&\left. \indent + \underbrace{\lambda \mathbb{E}_{\tilde{\theta},\tilde{W}_{1:n}\sim P_{\theta,W_{1:n}},Z_{1:n} \sim Q_{Z_{1:n}}}[{f(\tilde{\theta},\tilde{W}_{1:n}, Z_{1:n})}] - \log \mathbb{E}_{\tilde{\theta},\tilde{W}_{1:n} \sim P_{\theta,W_{1:n}},Z_{1:n} \sim Q_{Z_{1:n}}}[e^{\lambda f(\tilde{\theta},\tilde{W}_{1:n}, Z_{1:n})}]}_{\text{Term (3)}} \right\}.
\end{align*}
We first analyze Term (1),
\begin{align*}
\mathbb{E}_{\theta,W_{1:n},Z_{1:n} \sim P_{\theta,W_{1:n},Z_{1:n}}}[{f(\theta, W_{1:n}, Z_{1:n})}] &= \mathbb{E}_{\theta,W_{1:n},Z_{1:n} \sim P_{\theta,Z_{1:n}}P_{W_{1:n}|\theta,Z_{1:n}}}[{f(\theta, W_{1:n}, Z_{1:n})}] \\
&= \mathbb{E}_{\theta,Z_{1:n} \sim P_{\theta,Z_{1:n}}}\mathbb{E}_{W_{1:n} \sim P_{W_{1:n}|\theta,Z_{1:n}}}[{f(\theta, W_{1:n}, Z_{1:n})}]\\
&= \mathbb{E}_{\theta,Z_{1:n} \sim P_{\theta,Z_{1:n}}}\mathbb{E}_{W_{1:n} \sim P_{W_{1:n}|\theta,Z_{1:n}}}\left[\frac{1}{n} \sum_{i=1}^n \frac{1}{m} \sum_{j=1}^m \ell(W_i,S_{i,j})\right]\\
&= \mathbb{E}_{\theta,Z_{1:n} \sim P_{\theta,Z_{1:n}}}\left[\frac{1}{n} \sum_{i=1}^n \mathbb{E}_{W_{i} \sim P_{W_{i}|\theta,Z_{i}}} \frac{1}{m} \sum_{j=1}^m \ell(W_i,S_{i,j})\right]\\
&= \mathbb{E}_{\theta,Z_{1:n}}\left[L_{Z_{1:n}}(\theta)\right].
\end{align*}
For Term (2), we have
\begin{align*}
\mathbb{E}_{\tilde{\theta},\tilde{W}_{1:n}\sim P_{\theta,W_{1:n}},Z_{1:n} \sim Q_{Z_{1:n}}}[{f(\tilde{\theta},\tilde{W}_{1:n}, Z_{1:n})}] &= \mathbb{E}_{\tilde{\theta},\tilde{W}_{1:n}\sim P_{\theta,W_{1:n}},Z_{1:n} \sim Q_{Z_{1:n}}}\left[\frac{1}{n} \sum_{i=1}^n \frac{1}{m} \sum_{j=1}^m \ell(\tilde{W}_i,S_{i,j})\right]\\
&= \mathbb{E}_{\tilde{\theta},\tilde{W}_{1:n}\sim P_{\theta,W_{1:n}}}\left[\frac{1}{n} \sum_{i=1}^n \mathbb{E}_{Z_{i} \sim Q_{Z_i}}\frac{1}{m} \sum_{j=1}^m \ell(\tilde{W}_i,S_{i,j})\right]\\
&= \mathbb{E}_{\tilde{\theta},\tilde{W}_{1:n}\sim P_{\theta,W_{1:n}}}\left[\frac{1}{n} \sum_{i=1}^n \mathbb{E}_{\mu_i \sim \mathcal{U}}\mathbb{E}_{Z_{i}|\mu_i \sim \mu_i^{\otimes m}}\frac{1}{m} \sum_{j=1}^m \ell(\tilde{W}_i,S_{i,j})\right]\\
&= \mathbb{E}_{\tilde{\theta},\tilde{W}_{1:n}\sim P_{\theta,W_{1:n}}}\left[\frac{1}{n} \sum_{i=1}^n \mathbb{E}_{\mu_i \sim \mathcal{U}}\frac{1}{m} \sum_{j=1}^m \mathbb{E}_{S_{i,j} \sim \mu_i} \ell(\tilde{W}_i,S_{i,j})\right]\\
&= \mathbb{E}_{\tilde{\theta},\tilde{W}_{1:n}\sim P_{\theta,W_{1:n}}}\left[\frac{1}{n} \sum_{i=1}^n \mathbb{E}_{\mu_i \sim \mathcal{U}} L_{\mu_i}(\tilde{W}_i)\right]\\
&= \mathbb{E}_{\tilde{\theta},\tilde{W}_{1:n} \sim P_{\theta} P_{W_{1:n}|\theta}} \left[\frac{1}{n} \sum_{i=1}^n \mathbb{E}_{\mu_i \sim \mathcal{U}} L_{\mu_i}(\tilde{W}_i)\right]\\
&= \mathbb{E}_{\tilde{\theta} \sim P_{\theta}} \left[\frac{1}{n} \sum_{i=1}^n \mathbb{E}_{\tilde{W}_{i}\sim P_{W_{i}|\theta}} \mathbb{E}_{\mu_i \sim \mathcal{U}} L_{\mu_i}(\tilde{W}_i)\right]\\
&= \mathbb{E}_{\tilde{\theta} \sim P_{\theta}} \mathbb{E}_{\tilde{W}\sim P_{W|\theta}} \mathbb{E}_{\mu \sim \mathcal{U}} L_{\mu}(\tilde{W}).
\end{align*}
Recall that $\tilde{\theta}$ and $\tilde{W}_{1:n}$ are independent of the training tasks $Z_{1:n}$. In contrast, $\theta$ is learned from $Z_{1:n}$. The task-specific parameters $W_{1:n}$ in the population meta-risk with respect to $\mathcal{U}$ depend on both $\theta$ and the testing environment. Thus,
\begin{align*}
\mathbb{E}_{\tilde{\theta} \sim P_{\theta}} \mathbb{E}_{\tilde{W}\sim P_{W|\theta}} \mathbb{E}_{\mu \sim \mathcal{U}} L_{\mu}(\tilde{W}) &= \mathbb{E}_{Z_{1:n} \sim P_{Z_{1:n}}} \mathbb{E}_{{\theta} \sim P_{\theta|Z_{1:n}}} \mathbb{E}_{\mu \sim \mathcal{U}} \mathbb{E}_{Z\sim \mu^{\otimes m}} \mathbb{E}_{\tilde{W}\sim P_{W|Z,\theta}}  L_{\mu}(\tilde{W})\\
&= \mathbb{E}_{\theta,Z_{1:n}}[L_{\mathcal{U}}(\theta)].
\end{align*}
Next, since $\ell$ is $\sigma$-sub-Gaussian, Term (3) can be bounded as follows,
\begin{align*}
&\lambda \mathbb{E}_{\tilde{\theta},\tilde{W}_{1:n}\sim P_{\theta,W_{1:n}},Z_{1:n} \sim Q_{Z_{1:n}}}[{f(\tilde{\theta},\tilde{W}_{1:n}, Z_{1:n})}] - \log \mathbb{E}_{\tilde{\theta},\tilde{W}_{1:n} \sim P_{\theta,W_{1:n}},Z_{1:n} \sim Q_{Z_{1:n}}}[e^{\lambda f(\tilde{\theta},\tilde{W}_{1:n}, Z_{1:n})}]\\
&= \log \mathbb{E}_{\tilde{\theta},\tilde{W}_{1:n} \sim P_{\theta,W_{1:n}},Z_{1:n} \sim Q_{Z_{1:n}}}[e^{\lambda\mathbb{E}[f(\tilde{\theta},\tilde{W}_{1:n}, Z_{1:n})]-\lambda f(\tilde{\theta},\tilde{W}_{1:n}, Z_{1:n})}]\\
&\leq \frac{\lambda^2\sigma^2}{2nm}. \tag{$f(\tilde{\theta},\tilde{W}_{1:n}, Z_{1:n})$ is $\frac{\sigma}{\sqrt{nm}}$-sub-Gaussian}
\end{align*}
Putting all the previous steps together, we obtain
\begin{align*}
D(P_{\theta,W_{1:n},Z_{1:n}}\|P_{\theta,W_{1:n}}\otimes Q_{Z_{1:n}}) \geq \sup_{\lambda} \left\{ \lambda \mathbb{E}_{\theta,Z_{1:n}}\left[L_{Z_{1:n}}(\theta)-L_{\mathcal{U}}(\theta)\right] - \frac{\lambda^2\sigma^2}{2nm}  \right\},
\end{align*}
which implies 
\begin{align*}
\mathbb{E}_{\theta,Z_{1:n}}\left[L_{Z_{1:n}}(\theta)-L_{\mathcal{U}}(\theta)\right] &\leq \sqrt{\frac{2\sigma^2D(P_{\theta,W_{1:n},Z_{1:n}}\|P_{\theta,W_{1:n}}\otimes Q_{Z_{1:n}})}{nm}}\\
&\leq \sqrt{\frac{2\sigma^2(I(\theta,W_{1:n};Z_{1:n})+D(P_{Z_{1:n}}\|Q_{Z_{1:n}}))}{nm}}.
\end{align*}
\end{proof}

\subsection{Proof of Theorem \ref{thm_ml2}}

\thmmlt*

\begin{proof}
First, we let
\begin{align*}
f_i &:= f(\theta,\tau_i,Z_i) := \mathbbm{1}{\{\tau_{i}^{-}=\hat{\tau}\}} \mathbb{E}_{W_i^{-}|\theta,Z_{i}^{-}}L(W_i^{-},Z_{i}^{-})  - \mathbbm{1}{\{\tau_{i}^{+}=\hat{\tau}\}} \mathbb{E}_{W_i^{+}|\theta,Z_{i}^{+}}L(W_i^{+},Z_{i}^{+})
\end{align*}
where $L(W_i^{U_i},Z_{i}^{U_i})=\frac{1}{m}\sum_{j=1}^m \ell(W_i^{U_i},S^{U_i}_{i,j})$.
From the definition of $\text{gen}_{\text{sub}}$, we have
\begin{align*}
\text{gen}_{\text{sub}} &:=  \mathbb{E}_{\hat{\tau}} \mathbb{E}_{\theta,Z_{1:n}}\left[ \frac{1}{n_{\hat{\tau}}} \sum_{i=1}^n \mathbb{E}_{U_i}  \left[ U_i f(\theta,\tau_i,Z_i)  \right] \right] =  \mathbb{E}_{\hat{\tau}} \mathbb{E}_{Z_{1:n}}\left[ \frac{1}{n_{\hat{\tau}}} \sum_{i=1}^n \mathbb{E}_{\theta|Z_{1:n}} \mathbb{E}_{U_i}  \left[ U_i f_i  \right] \right].
\end{align*}
Denoting $(\tilde{f}_i,\tilde{U}_i)$ as an independent copy of $(f_i,U_i)$, the conditional mutual information can be written as
\begin{align*}
I(f_i;U_i|Z_{1:n}) &= D(P_{f_i,U_i|Z_{1:n}}\| P_{f_i|Z_{1:n}} \otimes P_{U_i|Z_{1:n}}) \\
&\geq \sup_{\lambda} \left\{  \mathbb{E}_{f_i,U_i|Z_{1:n}}[\lambda U_if_i] - \log  \mathbb{E}_{\tilde{f}_i,\tilde{U}_i|Z_{1:n}}[e^{\lambda \tilde{U}_i\tilde{f}_i}] \right\}. \tag{by Lemma \ref{lem_1}}
\end{align*}
Next, since $\ell$ is $\sigma$-sub-Gaussian, $L(W_i^{U_i},Z_i^{U_i})$ is $\frac{\sigma}{\sqrt{m}}$-sub-Gaussian. Thus, we have
\begin{align*}
 \log  \mathbb{E}_{\tilde{f}_i,\tilde{U}_i|Z_{1:n}}[e^{\lambda \tilde{U}_i\tilde{f}_i}] \leq \frac{\lambda^2\sigma^2}{2m}
\end{align*}
and
\begin{align*}
I(f_i;U_i|Z_{1:n}) &\geq \sup_{\lambda} \left\{ \lambda \mathbb{E}_{f_i,U_i|Z_{1:n}}[ U_if_i] - \frac{\lambda^2\sigma^2}{2m} \right\}= \sup_{\lambda} \left\{ \lambda \mathbb{E}_{\theta,U_i|Z_{1:n}}[ U_if_i] - \frac{\lambda^2\sigma^2}{2m} \right\}.
\end{align*}
The above imples
\begin{align*}
\mathbb{E}_{\theta,U_i|Z_{1:n}}[ U_if_i] \leq \sqrt{\frac{2\sigma^2 I(f_i;U_i|Z_{1:n})}{m}}.
\end{align*}
The analysis thus far concerns a fixed task $\tau_i$ with the sampled dataset $Z_i$. Taking the expectation with respect to $P_{Z_{1:n}}$, we obtain
\begin{align*}
\mathbb{E}_{Z_{1:n}}\left[ \frac{1}{n_{\hat{\tau}}} \sum_{i=1}^n \mathbb{E}_{\theta,U_i|Z_{1:n}} \left[ U_i f_i  \right] \right] \leq \mathbb{E}_{Z_{1:n}}\left[ \frac{1}{n_{\hat{\tau}}} \sum_{i=1}^n \sqrt{\frac{2\sigma^2 I(f_i;U_i|Z_{1:n})}{m}} \right].
\end{align*}
At this stage, we have established a task-wise generalization bound for task $\hat{\tau}$. To extend this result to a subtask generalization bound, we take the expectation with respect to the testing task distribution $\hat{\mathcal{T}}$. Therefore, we conclude that
\begin{align*}
\text{gen}_{\text{sub}} &:= \mathbb{E}_{\hat{\tau}} \mathbb{E}_{\theta,Z_{1:n}}\left[ \frac{1}{n_{\hat{\tau}}} \sum_{i=1}^n \mathbb{E}_{U_i}  \left[ U_i f_i  \right] \right] = \mathbb{E}_{\hat{\tau}} \mathbb{E}_{Z_{1:n}}\left[ \frac{1}{n_{\hat{\tau}}} \sum_{i=1}^n \mathbb{E}_{\theta,U_i|Z_{1:n}}  \left[ U_i f_i  \right] \right] \\
&\leq \mathbb{E}_{\hat{\tau}} \mathbb{E}_{Z_{1:n}}\left[ \frac{1}{n_{\hat{\tau}}} \sum_{i=1}^n \sqrt{\frac{2\sigma^2 I(f_i;U_i|Z_{1:n})}{m}} \right].
\end{align*}
\end{proof}

\section{Proofs for Section \ref{sec_mrl}}

\subsection{Proof of Theorem \ref{thm_mrrl1}} 

\thmmrlo*

\begin{proof}
The proof is similar to that of Theorem \ref{thm_ml1}. The key difference lies in how training data is obtained. In supervised learning, the data for each task is collected prior to training, whereas in RL it depends on the task-specific parameter $\phi_i$ and is therefore generated during training. Thus, the in-task objective in meta-RL is the expected discounted return, rather than the empirical loss. Denote $P_{\mathcal{M}_{1:n}}$ as the distribution over the $n$ MDPs $\mathcal{M}_{1:n}$ that results from the training environment $\mathcal{T}$, and $Q_{\mathcal{M}_{1:n}}$  as the distribution over the MDPs that results from the testing environment $\mathcal{U}$.
By Lemma \ref{lem_3} with $X=\mathcal{M}_{1:n}$, $Y=(\theta,\phi_{1:n})$, $P_{X}=P_{\mathcal{M}_{1:n}}$, and $Q_X=Q_{\mathcal{M}_{1:n}}$, it follows that
\begin{align*}
I(\theta, \phi_{1:n};\mathcal{M}_{1:n}) + D(P_{\mathcal{M}_{1:n}}\|Q_{\mathcal{M}_{1:n}}) = D(P_{\theta,\phi_{1:n},\mathcal{M}_{1:n}}\|P_{\theta,\phi_{1:n}}\otimes Q_{\mathcal{M}_{1:n}}).
\end{align*}
Let 
\begin{align*}
f(\theta, \phi_{1:n}, \mathcal{M}_{1:n}) := \frac{1}{n} \sum_{i=1}^n \mathbb{E}_{\omega \sim P_{\omega|\mathcal{M}_i,\phi_i}}\left[\sum_{j=0}^H \gamma^{j}r_j\right]
\end{align*}
Denoting $(\tilde{\theta},\tilde{\phi}_{1:n})$ as an independent copy of $(\theta,\phi_{1:n})$, we have
\begin{align*}
&D(P_{\theta,\phi_{1:n},\mathcal{M}_{1:n}}\|P_{\theta,\phi_{1:n}}\otimes Q_{\mathcal{M}_{1:n}}) \\
&\geq \sup_{\lambda} \left\{ \mathbb{E}_{\theta,\phi_{1:n},\mathcal{M}_{1:n} \sim P_{\theta,\phi_{1:n},\mathcal{M}_{1:n}}}[{\lambda f(\theta, \phi_{1:n}, \mathcal{M}_{1:n})}] - \log \mathbb{E}_{\tilde{\theta},\tilde{\phi}_{1:n} \sim P_{\theta,\phi_{1:n}},\mathcal{M}_{1:n} \sim Q_{\mathcal{M}_{1:n}}}[e^{\lambda f(\tilde{\theta},\tilde{\phi}_{1:n}, \mathcal{M}_{1:n})}] \right\} \tag{by Lemma \ref{lem_1}}\\
&= \sup_{\lambda} \left\{\lambda \big( \underbrace{\mathbb{E}_{\theta,\phi_{1:n},\mathcal{M}_{1:n} \sim P_{\theta,\phi_{1:n},\mathcal{M}_{1:n}}}[{f(\theta, \phi_{1:n}, \mathcal{M}_{1:n})}]}_{\text{Term (1)}} - \underbrace{\mathbb{E}_{\tilde{\theta},\tilde{\phi}_{1:n}\sim P_{\theta,\phi_{1:n}},\mathcal{M}_{1:n} \sim Q_{\mathcal{M}_{1:n}}}[{f(\tilde{\theta},\tilde{\phi}_{1:n}, \mathcal{M}_{1:n})}]}_{\text{Term (2)}} \big) \right. \\
&\left. \indent + \underbrace{\lambda \mathbb{E}_{\tilde{\theta},\tilde{\phi}_{1:n}\sim P_{\theta,\phi_{1:n}},\mathcal{M}_{1:n} \sim Q_{\mathcal{M}_{1:n}}}[{f(\tilde{\theta},\tilde{\phi}_{1:n}, \mathcal{M}_{1:n})}] - \log \mathbb{E}_{\tilde{\theta},\tilde{\phi}_{1:n} \sim P_{\theta,\phi_{1:n}},\mathcal{M}_{1:n} \sim Q_{\mathcal{M}_{1:n}}}[e^{\lambda f(\tilde{\theta},\tilde{\phi}_{1:n}, \mathcal{M}_{1:n})}]}_{\text{Term (3)}} \right\}.
\end{align*}
We first analyze Term (1),
\begin{align*}
&\mathbb{E}_{\theta,\phi_{1:n},\mathcal{M}_{1:n} \sim P_{\theta,\phi_{1:n},\mathcal{M}_{1:n}}}[{f(\theta, \phi_{1:n}, \mathcal{M}_{1:n})}] \\
&= \mathbb{E}_{\theta,\phi_{1:n},\mathcal{M}_{1:n} \sim P_{\theta,\mathcal{M}_{1:n}}P_{\phi_{1:n}|\theta,\mathcal{M}_{1:n}}}[{f(\theta, \phi_{1:n}, \mathcal{M}_{1:n})}] \\
&= \mathbb{E}_{\theta,\mathcal{M}_{1:n} \sim P_{\theta,\mathcal{M}_{1:n}}}\mathbb{E}_{\phi_{1:n} \sim P_{\phi_{1:n}|\theta,\mathcal{M}_{1:n}}}[{f(\theta, \phi_{1:n}, \mathcal{M}_{1:n})}]\\
&= \mathbb{E}_{\theta,\mathcal{M}_{1:n} \sim P_{\theta,\mathcal{M}_{1:n}}}\mathbb{E}_{\phi_{1:n} \sim P_{\phi_{1:n}|\theta,\mathcal{M}_{1:n}}}\left[\frac{1}{n} \sum_{i=1}^n \mathbb{E}_{\omega \sim P_{\omega|\mathcal{M}_i,\phi_i}}\left[\sum_{j=0}^H \gamma^{j}r_j\right]\right]\\
&= \mathbb{E}_{\theta,\mathcal{M}_{1:n} \sim P_{\theta,\mathcal{M}_{1:n}}}\left[\frac{1}{n} \sum_{i=1}^n \mathbb{E}_{\phi_{i} \sim P_{\phi_{i}|\theta,\mathcal{M}_{i}}} \mathbb{E}_{\omega \sim P_{\omega|\mathcal{M}_i,\phi_i}}\left[\sum_{j=0}^H \gamma^{j}r_j\right]\right]\\
&= \mathbb{E}_{\theta,\mathcal{M}_{1:n}}\left[J_{\mathcal{M}_{1:n}}(\theta)\right].
\end{align*}
For Term (2), we have
\begin{align*}
&\mathbb{E}_{\tilde{\theta},\tilde{\phi}_{1:n}\sim P_{\theta,\phi_{1:n}},\mathcal{M}_{1:n} \sim Q_{\mathcal{M}_{1:n}}}[{f(\tilde{\theta},\tilde{\phi}_{1:n}, \mathcal{M}_{1:n})}] \\
&= \mathbb{E}_{\tilde{\theta},\tilde{\phi}_{1:n}\sim P_{\theta,\phi_{1:n}},\mathcal{M}_{1:n} \sim Q_{\mathcal{M}_{1:n}}}\left[\frac{1}{n} \sum_{i=1}^n \mathbb{E}_{\omega \sim P_{\omega|\mathcal{M}_i,\tilde{\phi}_i}}\left[\sum_{j=0}^H \gamma^{j}r_j\right]\right]\\
&= \mathbb{E}_{\tilde{\theta},\tilde{\phi}_{1:n}\sim P_{\theta,\phi_{1:n}}}\left[\frac{1}{n} \sum_{i=1}^n \mathbb{E}_{\mathcal{M}_{i} \sim Q_{\mathcal{M}i}}\mathbb{E}_{\omega \sim P_{\omega|\mathcal{M}_i,\tilde{\phi}_i}}\left[\sum_{j=0}^H \gamma^{j}r_j\right]\right]\\
&= \mathbb{E}_{\tilde{\theta},\tilde{\phi}_{1:n}\sim P_{\theta,\phi_{1:n}}}\left[\frac{1}{n} \sum_{i=1}^n \mathbb{E}_{\mathcal{M}_i \sim \mathcal{U}}\mathbb{E}_{\omega \sim P_{\omega|\mathcal{M}_i,\tilde{\phi}_i}}\left[\sum_{j=0}^H \gamma^{j}r_j\right]\right]\\
&= \mathbb{E}_{\tilde{\theta},\tilde{\phi}_{1:n} \sim P_{\theta} P_{\phi_{1:n}|\theta}} \left[\frac{1}{n} \sum_{i=1}^n \mathbb{E}_{\mathcal{M}_i \sim \mathcal{U}}\mathbb{E}_{\omega \sim P_{\omega|\mathcal{M}_i,\tilde{\phi}_i}}\left[\sum_{j=0}^H \gamma^{j}r_j\right]\right]\\
&= \mathbb{E}_{\tilde{\theta} \sim P_{\theta}} \left[\frac{1}{n} \sum_{i=1}^n \mathbb{E}_{\tilde{\phi}_{i}\sim P_{\phi_{i}|\theta}} \mathbb{E}_{\mathcal{M}_i \sim \mathcal{U}}\mathbb{E}_{\omega \sim P_{\omega|\mathcal{M}_i,\tilde{\phi}_i}}\left[\sum_{j=0}^H \gamma^{j}r_j\right]\right]\\
&= \mathbb{E}_{\tilde{\theta} \sim P_{\theta}} \mathbb{E}_{\tilde{\phi}\sim P_{\phi|\theta}} \mathbb{E}_{\mathcal{M} \sim \mathcal{U}} \left[\frac{1}{n} \sum_{i=1}^n \mathbb{E}_{\omega \sim P_{\omega|\mathcal{M},\tilde{\phi}}}\left[\sum_{j=0}^H \gamma^{j}r_j\right]\right]\\
&= \mathbb{E}_{\tilde{\theta} \sim P_{\theta}} \mathbb{E}_{\tilde{\phi}\sim P_{\phi|\theta}} \mathbb{E}_{\mathcal{M} \sim \mathcal{U}} J_{\mathcal{M}}(\pi_{\tilde{\phi}}).
\end{align*}
Recall that $\tilde{\theta}$ and $\tilde{\phi}_{1:n}$ are independent of the training tasks $\mathcal{M}_{1:n}$. In contrast, $\theta$ is learned from $\mathcal{M}_{1:n}$. The task-specific parameters $\phi_{1:n}$ in the population meta-risk with respect to $\mathcal{U}$ depend on both $\theta$ and the testing environment. Thus,
\begin{align*}
\mathbb{E}_{\tilde{\theta} \sim P_{\theta}} \mathbb{E}_{\tilde{\phi}\sim P_{\phi|\theta}} \mathbb{E}_{\mathcal{M} \sim \mathcal{U}} J_{\mathcal{M}}(\pi_{\tilde{\phi}}) &= \mathbb{E}_{\mathcal{M}_{1:n} \sim P_{\mathcal{M}_{1:n}}} \mathbb{E}_{{\theta} \sim P_{\theta|\mathcal{M}_{1:n}}} \mathbb{E}_{\mathcal{M} \sim \mathcal{U}}  \mathbb{E}_{\tilde{\phi}\sim P_{\phi|\mathcal{M},\theta}} J_{\mathcal{M}}(\pi_{\tilde{\phi}})\\
&= \mathbb{E}_{\theta,\mathcal{M}_{1:n}}[J_{\mathcal{U}}(\theta)].
\end{align*}
Next, since $\sum_{j=0}^H \gamma^{j}r_j$ is bounded by $1/(1-\gamma)$, Term (3) can be bounded as follows,
\begin{align*}
&\lambda \mathbb{E}_{\tilde{\theta},\tilde{\phi}_{1:n}\sim P_{\theta,\phi_{1:n}},\mathcal{M}_{1:n} \sim Q_{\mathcal{M}_{1:n}}}[{f(\tilde{\theta},\tilde{\phi}_{1:n}, \mathcal{M}_{1:n})}] - \log \mathbb{E}_{\tilde{\theta},\tilde{\phi}_{1:n} \sim P_{\theta,\phi_{1:n}},\mathcal{M}_{1:n} \sim Q_{\mathcal{M}_{1:n}}}[e^{\lambda f(\tilde{\theta},\tilde{\phi}_{1:n}, \mathcal{M}_{1:n})}]\\
&= \log \mathbb{E}_{\tilde{\theta},\tilde{\phi}_{1:n} \sim P_{\theta,\phi_{1:n}},\mathcal{M}_{1:n} \sim Q_{\mathcal{M}_{1:n}}}[e^{\lambda\mathbb{E}[f(\tilde{\theta},\tilde{\phi}_{1:n}, \mathcal{M}_{1:n})]-\lambda f(\tilde{\theta},\tilde{\phi}_{1:n}, \mathcal{M}_{1:n})}]\\
&\leq \frac{\lambda^2}{2n(1-\gamma)^2}.
\end{align*}
Putting all the previous steps together, we obtain
\begin{align*}
D(P_{\theta,\phi_{1:n},\mathcal{M}_{1:n}}\|P_{\theta,\phi_{1:n}}\otimes Q_{\mathcal{M}_{1:n}}) \geq \sup_{\lambda} \left\{ \lambda \mathbb{E}_{\theta,\mathcal{M}_{1:n}}[J_{\mathcal{M}_{1:n}}(\theta)-J_{\mathcal{U}}(\theta)] - \frac{\lambda^2}{2n(1-\gamma)^2}  \right\},
\end{align*}
which implies 
\begin{align*}
\mathbb{E}_{\theta,\mathcal{M}_{1:n}}[J_{\mathcal{M}_{1:n}}(\theta)-J_{\mathcal{U}}(\theta)]  &\leq \sqrt{\frac{2D(P_{\theta,\phi_{1:n},\mathcal{M}_{1:n}}\|P_{\theta,\phi_{1:n}}\otimes Q_{\mathcal{M}_{1:n}})}{n(1-\gamma)^2}}\\
&\leq \sqrt{\frac{2I(\theta,\phi_{1:n};\mathcal{M}_{1:n})+2D(P_{\mathcal{M}_{1:n}}\|Q_{\mathcal{M}_{1:n}})}{n(1-\gamma)^2}}.
\end{align*}
\end{proof}

\subsection{Proof of Lemma \ref{lem_decompose}}

\lemdcom*

\begin{proof}
Recall that an MDP is defined by 
$\mathcal{M}(\mathcal{S}, \mathcal{A}, \mathcal{P}, \rho, r, \gamma,H)$. Let $\mathcal{P}, \rho, r$ and $\mathcal{P}^{\prime}, \rho^{\prime}, r^{\prime}$ denote the transition kernels, initial state distributions, and reward functions correspond to environments $P_{\mathcal{M}}$ and $Q_{\mathcal{M}}$ respectively. Denote $\mathcal{P}_{\times}(r^{\prime},s^{\prime}|s,a):=\mathcal{P}(s^{\prime}|s,a)\text{Prob}(r(s,a)=r^{\prime})$ as the joint transition kernel. Now, we have
\begin{align*}
D(P_{\mathcal{M}}\| Q_{\mathcal{M}}) &\leq \max_{\phi\in \Phi} \mathbb{E}_{\omega\sim P_{\mathcal{M}}}\left[ \log \frac{\rho(s_0) \prod_{h=0}^H \pi_{\phi}(a_h|s_h) \mathcal{P}_{\times}(s_{h+1},r_h|s_h,a_h)}{\rho^{\prime}(s_0) \prod_{h=0}^H \pi_{\phi}(a_h|s_h) \mathcal{P}^{\prime}_{\times}(s_{h+1},r_h|s_h,a_h)} \right]\\
&= D(\rho\|\rho^{\prime}) + \sum_{h=0}^H \max_{\phi\in \Phi} \mathbb{E}_{\omega\sim P_{\mathcal{M}}}\left[ \log \frac{\pi_{\phi}(a_h|s_h) \mathcal{P}_{\times}(s_{h+1},r_h|s_h,a_h)}{\pi_{\phi}(a_h|s_h) \mathcal{P}^{\prime}_{\times}(s_{h+1},r_h|s_h,a_h)} \right].
\end{align*}
Denote $d^{\pi_{\phi}}_h(\cdot,\cdot)$ as the state-action marginal distribution of policy $\pi_{\phi}$ at timestamp $h$. We obtain
\begin{align*}
&\sum_{h=0}^H \max_{\phi\in \Phi} \mathbb{E}_{\omega\sim P_{\mathcal{M}}}\left[ \log \frac{\pi_{\phi}(a_h|s_h) \mathcal{P}_{\times}(s_{h+1},r_h|s_h,a_h)}{\pi_{\phi}(a_h|s_h) \mathcal{P}^{\prime}_{\times}(s_{h+1},r_h|s_h,a_h)} \right] \\
= &\sum_{h=0}^H \max_{\phi\in \Phi} \mathbb{E}_{(s_h,a_h)\sim d^{\pi_{\phi}}_h} \mathbb{E}_{(s_{h+1},r_h)\sim \mathcal{P}_{\times}}\left[ \log \frac{ \mathcal{P}_{\times}(s_{h+1},r_h|s_h,a_h)}{\mathcal{P}^{\prime}_{\times}(s_{h+1},r_h|s_h,a_h)} \right]\\
= &\sum_{h=0}^H \max_{\phi\in \Phi} \mathbb{E}_{(s_h,a_h)\sim d^{\pi_{\phi}}_h} \left[ \mathbb{E}_{s_{h+1}\sim \mathcal{P}}\left[ \log \frac{ \mathcal{P}(s_{h+1}|s_h,a_h)}{\mathcal{P}^{\prime}(s_{h+1}|s_h,a_h)} \right] + \mathbb{E}_{r_{h}\sim {r}} \left[ \log \frac{ \text{Prob}(r(s_h,a_h)=r_h) }{\text{Prob}(r^{\prime}(s_h,a_h)=r_h)} \right]\right]\\
\leq &\sum_{h=0}^H \max_{(s,a)\in \mathcal{S}\times \mathcal{A}} D(\mathcal{P}(\cdot|s,a)\| \mathcal{P}^{\prime}(\cdot|s,a)) + \sum_{h=0}^H\max_{(s,a)\in \mathcal{S}\times \mathcal{A}} D(r(s,a)\|r^{\prime}(s,a))  \\
= &(H+1) \max_{(s,a)\in \mathcal{S}\times \mathcal{A}} D(\mathcal{P}(\cdot|s,a)\| \mathcal{P}^{\prime}(\cdot|s,a)) + (H+1)\max_{(s,a)\in \mathcal{S}\times \mathcal{A}} D(r(s,a)\|r^{\prime}(s,a)).
\end{align*}
Putting everthing together, we conclude that
\begin{align*}
D(P_{\mathcal{M}}\| Q_{\mathcal{M}}) \leq D(\rho\| \rho^{\prime}) + (H+1) \max_{(s,a)\in \mathcal{S}\times \mathcal{A}}( D(\mathcal{P}(\cdot|s,a)\| \mathcal{P}^{\prime}(\cdot|s,a)) + D(r(s,a)\|r^{\prime}(s,a))).
\end{align*}
\end{proof}

\subsection{Proof of Theorem \ref{thm_mrrl7}}

\thmmrlt*

\begin{proof}
First, we let
\begin{align*}
f_i &:= f(\theta,\mathcal{M}_i) := \mathbbm{1}{\{\mathcal{M}_{i}^{-}=\hat{\mathcal{M}}\}} \mathbb{E}_{\phi_i^{-}|\theta,\mathcal{M}_{i}^{-}}J(\pi_{\phi_i^{-}},\mathcal{M}_{i}^{-})  - \mathbbm{1}{\{\mathcal{M}_{i}^{+}=\hat{\mathcal{M}}\}} \mathbb{E}_{\phi_i^{+}|\theta,\mathcal{M}_{i}^{+}}J(\pi_{\phi_i^{+}},\mathcal{M}_{i}^{+})
\end{align*}
where $J(\pi_{\phi_i^{U_i}},\mathcal{M}_{i}^{U_i})= \mathbb{E}_{\omega \sim P_{\omega|\mathcal{M}_{i}^{U_i},\phi_i^{U_i}}}\left[ \sum_{j=0}^H \gamma^j r_j \right]$.
From the definition of $\text{gen}_{\text{sub}}$, we have
\begin{align*}
\text{gen}_{\text{sub}} &:=  \mathbb{E}_{\hat{\mathcal{M}}} \mathbb{E}_{\theta,\mathcal{M}_{1:n}}\left[ \frac{1}{n_{\hat{\tau}}} \sum_{i=1}^n \mathbb{E}_{U_i}  \left[ U_i f(\theta,\mathcal{M}_i)  \right] \right] =  \mathbb{E}_{\hat{\mathcal{M}}} \mathbb{E}_{\mathcal{M}_{1:n}}\left[ \frac{1}{n_{\hat{\tau}}} \sum_{i=1}^n \mathbb{E}_{\theta|\mathcal{M}_{1:n}} \mathbb{E}_{U_i}  \left[ U_i f_i  \right] \right].
\end{align*}
Denoting $(\tilde{f}_i,\tilde{U}_i)$ as an independent copy of $(f_i,U_i)$, the conditional mutual information can be written as
\begin{align*}
I(f_i;U_i|\mathcal{M}_{1:n}) &= D(P_{f_i,U_i|\mathcal{M}_{1:n}}\| P_{f_i|\mathcal{M}_{1:n}} \otimes P_{U_i|\mathcal{M}_{1:n}}) \\
&\geq \sup_{\lambda} \left\{  \mathbb{E}_{f_i,U_i|\mathcal{M}_{1:n}}[\lambda U_if_i] - \log  \mathbb{E}_{\tilde{f}_i,\tilde{U}_i|\mathcal{M}_{1:n}}[e^{\lambda \tilde{U}_i\tilde{f}_i}] \right\}. \tag{by Lemma \ref{lem_1}}
\end{align*}
Next, since $J(\pi_{\phi_i^{U_i}},\mathcal{M}_{i}^{U_i})$ is bounded by $1/(1-\gamma)$, we have
\begin{align*}
 \log  \mathbb{E}_{\tilde{f}_i,\tilde{U}_i|\mathcal{M}_{1:n}}[e^{\lambda \tilde{U}_i\tilde{f}_i}] \leq \frac{\lambda^2}{2(1-\gamma)^2}
\end{align*}
and
\begin{align*}
I(f_i;U_i|\mathcal{M}_{1:n}) &\geq \sup_{\lambda} \left\{ \lambda \mathbb{E}_{f_i,U_i|\mathcal{M}_{1:n}}[ U_if_i] - \frac{\lambda^2}{2(1-\gamma)^2} \right\}= \sup_{\lambda} \left\{ \lambda \mathbb{E}_{\theta,U_i|\mathcal{M}_{1:n}}[ U_if_i] - \frac{\lambda^2}{2(1-\gamma)^2} \right\}.
\end{align*}
The above imples
\begin{align*}
\mathbb{E}_{\theta,U_i|\mathcal{M}_{1:n}}[ U_if_i] \leq \sqrt{\frac{2 I(f_i;U_i|\mathcal{M}_{1:n})}{(1-\gamma)^2}}.
\end{align*}
The analysis thus far concerns a fixed task $\mathcal{M}_i$. Taking the expectation with respect to $P_{\mathcal{M}_{1:n}}$, we obtain
\begin{align*}
\mathbb{E}_{\mathcal{M}_{1:n}}\left[ \frac{1}{n_{\hat{\mathcal{M}}}} \sum_{i=1}^n \mathbb{E}_{\theta,U_i|\mathcal{M}_{1:n}} \left[ U_i f_i  \right] \right] \leq \mathbb{E}_{\mathcal{M}_{1:n}}\left[ \frac{1}{n_{\hat{\mathcal{M}}}} \sum_{i=1}^n \sqrt{\frac{2 I(f_i;U_i|\mathcal{M}_{1:n})}{(1-\gamma)^2}} \right].
\end{align*}
At this stage, we have established a task-wise generalization bound for task $\hat{\mathcal{M}}$. To extend this result to a subtask generalization bound, we take the expectation with respect to the testing task distribution $\hat{\mathcal{T}}$. Therefore, we conclude that
\begin{align*}
\text{gen}_{\text{sub}} &:= \mathbb{E}_{\hat{\mathcal{M}}} \mathbb{E}_{\theta,\mathcal{M}_{1:n}}\left[ \frac{1}{n_{\hat{\mathcal{M}}}} \sum_{i=1}^n \mathbb{E}_{U_i}  \left[ U_i f_i  \right] \right] = \mathbb{E}_{\hat{\mathcal{M}}} \mathbb{E}_{\mathcal{M}_{1:n}}\left[ \frac{1}{n_{\hat{\mathcal{M}}}} \sum_{i=1}^n \mathbb{E}_{\theta,U_i|\mathcal{M}_{1:n}}  \left[ U_i f_i  \right] \right] \\
&\leq \mathbb{E}_{\hat{\mathcal{M}}} \mathbb{E}_{\mathcal{M}_{1:n}}\left[ \frac{1}{n_{\hat{\mathcal{M}}}} \sum_{i=1}^n \sqrt{\frac{2 I(f_i;U_i|\mathcal{M}_{1:n})}{(1-\gamma)^2}} \right].
\end{align*}
\end{proof}

\subsection{Proof of Theorem \ref{thm_mrl5}}

\thmmrlff*

\begin{proof}
Let $\{ \theta^m\}_{m=0}^M$ denote the sequence of meta-parameters, and $\{ \phi^k_{1:n}\}_{k=0}^{MT}$ denote the sequence of task-specific parameters. Their final outputs are denoted by $(\theta,\phi_{1:n})$. Denote $\nabla J^m_{\text{meta}}=\alpha_m \frac{1}{b} \sum_{\mathcal{M}_i\in \mathcal{B}_m} \nabla J_{\mathcal{M}_i}(\theta^m) +\xi_m$ and $\nabla J^{{k}}_{\text{inner}}=\begin{bmatrix}
    \beta_k \nabla J(\pi_{\phi_{1}^k},\mathcal{M}_1) + \zeta_k \\
    \vdots \\
    \beta_k \nabla J(\pi_{\phi_{n}^k},\mathcal{M}_n) + \zeta_k
\end{bmatrix}$.
For any $Tm+t \in [TM]$, we have the Markov chain
\begin{align*}
\mathcal{M}_{1:n} \rightarrow (\theta^m,\phi_{1:n}^{{Tm+t}},\nabla J^m_{\text{meta}},\nabla J^{{Tm+t}}_{\text{inner}}) \rightarrow (\theta,\phi_{1:n}).
\end{align*}
By the data processing inequality (Lemma \ref{lem_2}), we have
\begin{align}
&I(\theta,\phi_{1:n};\mathcal{M}_{1:n})\nonumber \\
&= I(\theta^{M},\phi^{TM}_{1:n};\mathcal{M}_{1:n})\nonumber \\
&= I(\theta^{M-1}+\nabla J^{M-1}_{\text{meta}},\phi^{TM}_{1:n};\mathcal{M}_{1:n})\nonumber \\
&\leq I(\theta^{M-1},\nabla J^{M-1}_{\text{meta}},\phi^{TM}_{1:n};\mathcal{M}_{1:n})\nonumber \\
&= I(\theta^{M-1},\phi^{TM}_{1:n};\mathcal{M}_{1:n}) + I(\nabla J^{M-1}_{\text{meta}};\mathcal{M}_{1:n}|\theta^{M-1},\phi^{TM}_{1:n})\nonumber \\
&= I(\theta^{M-1},\phi^{TM-1}_{1:n}+\nabla J^{TM-1}_{\text{inner}};\mathcal{M}_{1:n}) + I(\nabla J^{M-1}_{\text{meta}};\mathcal{M}_{1:n}|\theta^{M-1},\phi^{TM}_{1:n}) \nonumber \\
&\leq I(\theta^{M-1},\phi^{TM-1}_{1:n},\nabla J^{TM-1}_{\text{inner}};\mathcal{M}_{1:n}) + I(\nabla J^{M-1}_{\text{meta}};\mathcal{M}_{1:n}|\theta^{M-1},\phi^{TM}_{1:n}) \nonumber \\
&= I(\theta^{M-1},\phi^{TM-1}_{1:n};\mathcal{M}_{1:n}) + I(\nabla J^{TM-1}_{\text{inner}};\mathcal{M}_{1:n}|\theta^{M-1},\phi^{TM-1}_{1:n}) + I(\nabla J^{M-1}_{\text{meta}};\mathcal{M}_{1:n}|\theta^{M-1},\phi^{TM}_{1:n}) \nonumber \\
&\leq I(\theta^{M-1},\phi^{TM-T}_{1:n};\mathcal{M}_{1:n}) + I(\nabla J^{M-1}_{\text{meta}};\mathcal{M}_{1:n}|\theta^{M-1},\phi^{TM}_{1:n}) + \sum_{t=1}^T I(\nabla J^{TM-t}_{\text{inner}};\mathcal{M}_{1:n}|\theta^{M-1},\phi^{TM-t}_{1:n}) \nonumber \\
&\leq I(\theta^{0},\phi^{0}_{1:n};\mathcal{M}_{1:n}) +  \sum_{m=1}^M I(\nabla J^{m-1}_{\text{meta}};\mathcal{M}_{1:n}|\theta^{m-1},\phi^{Tm}_{1:n}) + \sum_{m=1}^M \sum_{t=1}^T I(\nabla J^{Tm-t}_{\text{inner}};\mathcal{M}_{1:n}|\theta^{m-1},\phi^{Tm-t}_{1:n}) \nonumber \\
&= \sum_{m=1}^M I(\nabla J^{m-1}_{\text{meta}};\mathcal{M}_{1:n}|\theta^{m-1},\phi^{Tm}_{1:n}) + \sum_{m=1}^M \sum_{t=1}^T I(\nabla J^{Tm-t}_{\text{inner}};\mathcal{M}_{1:n}|\theta^{m-1},\phi^{Tm-t}_{1:n}). \label{eqn_2mi}
\end{align}
Note that $\theta^{m-1}$ and $\phi^{Tm}_{1:n}$ are random variables, and the conditional mutual information is defined as an expectation over their joint distribution. We use the notation $I^{\theta^{m-1},\phi^{Tm}_{1:n}}(\nabla J^{m-1}_{\text{meta}};\mathcal{M}_{1:n})$ to denote the conditional mutual information given fixed realizations of $\theta^{m-1}$ and $\phi^{Tm}_{1:n}$ (i.e., without taking the expectation).
We now analyze each term in \cref{eqn_2mi} individually. For the first term, we have
\begin{align*}
\text{cov}(\nabla J^{m}_{\text{meta}})&= \text{cov}\left(\alpha_m \frac{1}{b} \sum_{\mathcal{M}_i \in \mathcal{B}_m} \nabla J_{\mathcal{M}_i}(\theta^m) +\xi_m\right)\\
&= \alpha_m^2 \tilde{\Sigma}_m + \tilde{\kappa}^2_m \mathbf{1}_d
\end{align*}
where
\begin{align*}
\tilde{\Sigma}_m = \text{cov}\left(\frac{1}{b} \sum_{\mathcal{M}_i \in \mathcal{B}_m} \nabla J_{\mathcal{M}_i}(\theta^m)\right).
\end{align*}
Therefore,
\begin{align*}
&\sum_{m=1}^M I^{\theta^{m-1},\phi^{Tm}_{1:n}}(\nabla J^{m-1}_{\text{meta}};\mathcal{M}_{1:n}) \\
&= \sum_{m=0}^{M-1} I^{\theta^{m},\phi^{T(m+1)}_{1:n}}(\alpha_m \frac{1}{b} \sum_{\mathcal{M}_i \in \mathcal{B}_m} \nabla J_{\mathcal{M}_i}(\theta^m) +\xi_m;\mathcal{M}_{1:n})\\
&= \sum_{m=0}^{M-1} H(\alpha_m \frac{1}{b} \sum_{\mathcal{M}_i \in \mathcal{B}_m} \nabla J_{\mathcal{M}_i}(\theta^m) +\xi_m|\theta^{m},\phi^{T(m+1)}_{1:n}) - \sum_{m=0}^{M-1} H(\alpha_m \frac{1}{b} \sum_{\mathcal{M}_i \in \mathcal{B}_m} \nabla J_{\mathcal{M}_i}(\theta^m) +\xi_m|\theta^{m},\phi^{T(m+1)}_{1:n},\mathcal{M}_{1:n})\\
&\leq \sum_{m=0}^{M-1} H(\alpha_m \frac{1}{b} \sum_{\mathcal{M}_i \in \mathcal{B}_m} \nabla J_{\mathcal{M}_i}(\theta^m) +\xi_m|\theta^{m},\phi^{T(m+1)}_{1:n}) - \sum_{m=0}^{M-1} H(\xi_m) \tag{independent $\xi_m$}\\
&= \sum_{m=0}^{M-1} H(\alpha_m \frac{1}{b} \sum_{\mathcal{M}_i \in \mathcal{B}_m} \nabla J_{\mathcal{M}_i}(\theta^m) +\xi_m|\theta^{m},\phi^{T(m+1)}_{1:n}) - \sum_{m=0}^{M-1} (\frac{d}{2}\log (2\pi e)+ \frac{d}{2}\log(\tilde{\kappa}^2_m )) \tag{entropy of $d$-dimensional Gaussian vector}\\
&\leq \sum_{m=0}^{M-1} ( \frac{d}{2}\log(2\pi e) + \frac{1}{2}\log(\text{det}(\alpha_m^2 \tilde{\Sigma}_m + \tilde{\kappa}^2_m \mathbf{1}_d)) - \frac{d}{2}\log(2\pi e\tilde{\kappa}^2_m)) \tag{by Lemma \ref{lem_ent}}\\
&=  \sum_{m=0}^{M-1} \frac{1}{2}\log(\text{det}(\frac{\alpha_m^2}{\tilde{\kappa}^2_m}\tilde{\Sigma}_m+\mathbf{1}_d)).
\end{align*}
For the second term in \cref{eqn_2mi}, we have
\begin{align*}
\text{cov}(\nabla J^{k}_{\text{inner}})&= \text{cov}\left({\beta}_k \nabla \bm{J}^k + \bm{\zeta}_k \right)= \beta_k^2 {\Sigma}_k + {\kappa}^2_k \mathbf{1}_{nd}
\end{align*}
where
\begin{align*}
{\Sigma}_k = \text{cov}\left(\nabla \bm{J}^k\right),
\nabla \bm{J}^k=
\begin{bmatrix}
    \nabla J(\pi_{\phi_{1}^k},\mathcal{M}_1) \\
    \vdots \\
    \nabla J(\pi_{\phi_{n}^k},\mathcal{M}_n)
\end{bmatrix}, \ \ \text{and} \ \
\bm{\zeta}_k=
\begin{bmatrix}
    \zeta_k \\
    \vdots \\
    \zeta_k
\end{bmatrix}.
\end{align*}
Thus,
\begin{align*}
&\sum_{m=1}^M \sum_{t=1}^T I^{\theta^{m-1},\phi^{Tm-t}_{1:n}}(\nabla J^{Tm-t}_{\text{inner}};\mathcal{M}_{1:n})\\
&= \sum_{m=0}^{M-1} \sum_{t=1}^T I^{\theta^{m},\phi^{T(m+1)-t}_{1:n}}( {\beta}_{T(m+1)-t} \nabla \bm{J}^{T(m+1)-t} + \bm{\zeta}_{T(m+1)-t};\mathcal{M}_{1:n})\\
&= \sum_{m=0}^{M-1}  \sum_{t=1}^T H({\beta}_{T(m+1)-t} \nabla \bm{J}^{T(m+1)-t} + \bm{\zeta}_{T(m+1)-t}|\theta^{m},\phi^{T(m+1)-t}_{1:n}) \\
&\indent - \sum_{m=0}^{M-1}  \sum_{t=1}^T H({\beta}_{T(m+1)-t} \nabla \bm{J}^{T(m+1)-t} + \bm{\zeta}_{T(m+1)-t}|\theta^{m},\phi^{T(m+1)-t}_{1:n},\mathcal{M}_{1:n})\\
&= \sum_{m=0}^{M-1}  \sum_{t=1}^T H({\beta}_{T(m+1)-t} \nabla \bm{J}^{T(m+1)-t} + \bm{\zeta}_{T(m+1)-t}|\theta^{m},\phi^{T(m+1)-t}_{1:n})  - \sum_{m=0}^{M-1}  \sum_{t=1}^T H(\bm{\zeta}_{T(m+1)-t}|\theta^{m},\phi^{T(m+1)-t}_{1:n},\mathcal{M}_{1:n})\\
&= \sum_{m=0}^{M-1}  \sum_{t=1}^T H({\beta}_{T(m+1)-t} \nabla \bm{J}^{T(m+1)-t} + \bm{\zeta}_{T(m+1)-t}|\theta^{m},\phi^{T(m+1)-t}_{1:n})  - \sum_{m=0}^{M-1}  \sum_{t=1}^T  \frac{nd}{2}\log (2\pi e)+ \frac{nd}{2}\log({\kappa}^2_{T(m+1)-t} )\\
&\leq \sum_{m=0}^{M-1} ( \frac{nd}{2}\log(2\pi e) + \frac{1}{2}\log(\text{det}(\beta_{T(m+1)-t)}^2 {\Sigma}_{T(m+1)-t)} + {\kappa}^2_{T(m+1)-t)} \mathbf{1}_{nd})) - \frac{nd}{2}\log(2\pi e{\kappa}^2_{T(m+1)-t)}) \tag{by Lemma \ref{lem_ent}}\\
&=  \sum_{m=0}^{M-1} \frac{1}{2}\log(\text{det}(\frac{\beta_{T(m+1)-t)}^2}{{\kappa}^2_{T(m+1)-t)}}{\Sigma}_{T(m+1)-t)}+\mathbf{1}_{nd})).
\end{align*}
Putting them together with \cref{eqn_2mi}, we have
\begin{align}
&I(\theta,\phi_{1:n};\mathcal{M}_{1:n})= \sum_{m=0}^{M-1} \mathbb{E}_{\theta^{m},\phi^{T(m+1)}_{1:n}}\left[  I^{\theta^{m},\phi^{T(m+1)}_{1:n}}(\alpha_m \frac{1}{b} \sum_{\mathcal{M}_i \in \mathcal{B}_m} \nabla J_{\mathcal{M}_i}(\theta^m) +\xi_m;\mathcal{M}_{1:n})\right] \nonumber \\
&\indent + \sum_{m=0}^{M-1} \sum_{t=1}^T \mathbb{E}_{\theta^{m},\phi^{T(m+1)-t}_{1:n}}\left[I^{\theta^{m},\phi^{T(m+1)-t}_{1:n}}(\nabla J^{T(m+1)-t}_{\text{inner}};\mathcal{M}_{1:n})\right] \nonumber \\
&\leq \frac{1}{2}  \sum_{m=0}^{M-1} \mathbb{E}_{\theta^m}\left[ \log\left(\textnormal{det}\left(\frac{\alpha_m^2}{\tilde{\kappa}^2_m}\tilde{\Sigma}_m+\mathbf{1}_d\right)\right)\right] \nonumber  \\
&\indent +  \frac{1}{2} \sum_{m=0}^{M-1}\sum_{t=1}^T \mathbb{E}_{\theta^m,\phi_{1:n}^{T(m+1)-t}}\left[\log\left(\textnormal{det}\left(\frac{\beta_{T(m+1)-t}^2}{{\kappa}^2_{T(m+1)-t}}{\Sigma}_{T(m+1)-t}+\mathbf{1}_{nd}\right)\right)\right] \label{eqn_f2}
\end{align}
From Theorem \ref{thm_mrrl1}, we know 
\begin{align*}
\mathbb{E}_{\theta,\mathcal{M}_{1:n}}[J_{\mathcal{M}_{1:n}}(\theta)-J_{\mathcal{U}}(\theta)] \leq \sqrt{\frac{2I(\theta,\phi_{1:n};\mathcal{M}_{1:n})+2D(P_{\mathcal{M}_{1:n}}\|Q_{\mathcal{M}_{1:n}})}{n(1-\gamma)^2}}.
\end{align*}
Combining this with \cref{eqn_f2} completes the proof.
\end{proof}

\end{document}